\documentclass[10pt,conference,letterpaper]{IEEEtran}

\usepackage{times}
\usepackage{amssymb}
\usepackage{pifont}

\usepackage{amsthm}
\usepackage{amsmath}
\usepackage{graphicx}
\usepackage{epsfig}
\usepackage{epstopdf}
\usepackage{caption}
\DeclareCaptionType{copyrightbox}
\usepackage{subcaption}
\usepackage[ruled, vlined]{algorithm2e}
\usepackage{array}
\usepackage{enumerate}
\usepackage{comment}
\usepackage{multicol}
\usepackage{multirow}
\usepackage{balance}
\usepackage[table]{xcolor}
\usepackage{url}
\usepackage{verbatim}

\usepackage{microtype}
\let\emptyset\varnothing

\newcommand{\eat}[1]{}

\newcommand{\algname}[1]{\textsf{\footnotesize #1}}
\newcommand{\salgname}[1]{\textsf{\scriptsize #1}}

\DeclareMathOperator*{\argmin}{arg\,min}

\newcommand{\topk}{top-\emph{k}}

{
\theoremstyle{definition}
\newtheorem{definition}{Definition}

\newtheorem{theorem}{Theorem}

\newtheorem{example}{Example}
\newtheorem{property}{Property}

\newtheorem{myrule}{Rule}
\newtheorem{strategy}{Strategy}
}

\newcommand{\obj}[1]{\textsf{\footnotesize #1}}
\newcommand{\paircomp}[3]{\obj{#1}\ensuremath{?_{#3}}\obj{#2}}
\newcommand{\compxyc}{\paircomp{x}{y}{c}}

\newcommand{\domO}[2]{\obj{#1}\ensuremath{\succ}\obj{#2}}

\newcommand{\indiffO}[2]{\obj{#1}\ensuremath{\sim}\obj{#2}}
\newcommand{\domC}[3]{\obj{#1}\ensuremath{\succ_{#3}}\obj{#2}}
\newcommand{\worseC}[3]{\obj{#1}\ensuremath{\prec_{#3}}\obj{#2}}
\newcommand{\indiffC}[3]{\obj{#1}\ensuremath{\sim_{#3}}\obj{#2}}
\newcommand{\rlt}{\ensuremath{rlt}}

\title{\huge Crowdsourcing Pareto-Optimal Object Finding\\ by Pairwise Comparisons\vspace{-3mm}}

\author{%
{Abolfazl Asudeh, Gensheng Zhang, Naeemul Hassan, Chengkai Li, Gergely V. Zaruba}\\
\fontsize{10}{10}\selectfont\itshape
Department of Computer Science and Engineering, The University of Texas at Arlington\\
}

\begin{document}
\maketitle

\thispagestyle{plain}
\pagestyle{plain}

\begin{abstract}
This is the first study on crowdsourcing Pareto-optimal object finding, which has applications in public opinion collection, group decision making, and information exploration. Departing from prior studies on crowdsourcing skyline and ranking queries, it considers the case where objects do not have explicit attributes and preference relations on objects are strict partial orders. The partial orders are derived by aggregating crowdsourcers' responses to pairwise comparison questions. The goal is to find all Pareto-optimal objects by the fewest possible questions. It employs an iterative question-selection framework. Guided by the principle of eagerly identifying non-Pareto optimal objects, the framework only chooses candidate questions which must satisfy three conditions. This design is both sufficient and efficient, as it is proven to find a short terminal question sequence. The framework is further steered by two ideas---macro-ordering and micro-ordering. By different micro-ordering heuristics, the framework is instantiated into several algorithms with varying power in pruning questions. Experiment results using both real crowdsourcing marketplace and simulations exhibited not only orders of magnitude reductions in questions when compared with a brute-force approach, but also close-to-optimal performance from the most efficient instantiation.
\end{abstract}

\section{Introduction}\vspace{-1mm}

The growth of user engagement and functionality in crowdsourcing platforms
has made computationally challenging tasks unprecedentedly convenient.  The
subject of our study is one such task---crowdsourcing \emph{Pareto-optimal
object finding}.  Consider a set of objects $O$ and a set of criteria $C$
for comparing the objects.  An object $\obj{x}$$\in$$O$ is
\emph{Pareto-optimal} if and only if \obj{x} is not dominated by any other
object, i.e., $\nexists$$\obj{y}$$\in$$O$ such that \domO{y}{x}.  An object
\obj{y} dominates \obj{x} (denoted \domO{y}{x}) if and only if \obj{x} is
not better than \obj{y} by any criterion and \obj{y} is better than \obj{x}
by at least one criterion, i.e., $\forall c$$\in$$C : \obj{x}$$\nsucc_c$$\obj{y}$
and $\exists c$$\in$$C : \domC{y}{x}{c}$.  If \obj{x} and \obj{y} do not
dominate each other (i.e., $\obj{x}$$\nsucc$$\obj{y}$ and $\obj{y}$$\nsucc$$\obj{x}$),
we denote it by \indiffO{x}{y}.
The \emph{preference (better-than) relation} $P_c$ (also denoted $\succ_c$)
for each $c$$\in$$C$ is a binary relation subsumed by $O$$\times$$O$, in
which a tuple $(\obj{x},\obj{y})$$\in$$P_c$ (also denoted \domC{x}{y}{c}) is
interpreted as ``\obj{x} is better than (preferred over) \obj{y} with regard
to criterion $c$''.  Hence, if $(\obj{x},\obj{y})$$\notin$$P_c$ (also denoted
\obj{x}$\nsucc_c$\obj{y}), \obj{x} is not better than \obj{y} by criterion
$c$.  We say \obj{x} and \obj{y} are \emph{indifferent} regarding $c$ (denoted
\indiffC{x}{y}{c}), if $(\obj{x},\obj{y})$$\notin$$P_c \wedge (\obj{y},\obj{x})$$\notin$$P_c$, i.e.,
``\obj{x} and \obj{y} are equally good or incomparable with regard to $c$.''
We consider the setting where each $P_c$ is a \emph{strict partial order} as opposed to
a bucket order~\cite{bucketorder-fagin-pods04} or a total order,
i.e., $P_c$ is irreflexive ($\forall \obj{x} : (\obj{x},\obj{x}) \notin P_c$) and
transitive ($\forall \obj{x},\obj{y} : (\obj{x},\obj{y})$$\in$$P_c$$\wedge$$(\obj{y},\obj{z})$$\in$$P_c$$\Rightarrow$$(\obj{x},\obj{z})$$\in$$P_c$), which together
imply asymmetry ($\forall \obj{x},\obj{y} : (\obj{x},\obj{y})$$\in$$P_c \Rightarrow (\obj{y},\obj{x})$$\notin$$P_c$).

\begin{figure}[t]
    \begin{subfigure}[b]{\linewidth}
    \centering
    \vspace{-1mm}
        \includegraphics[width=44mm]{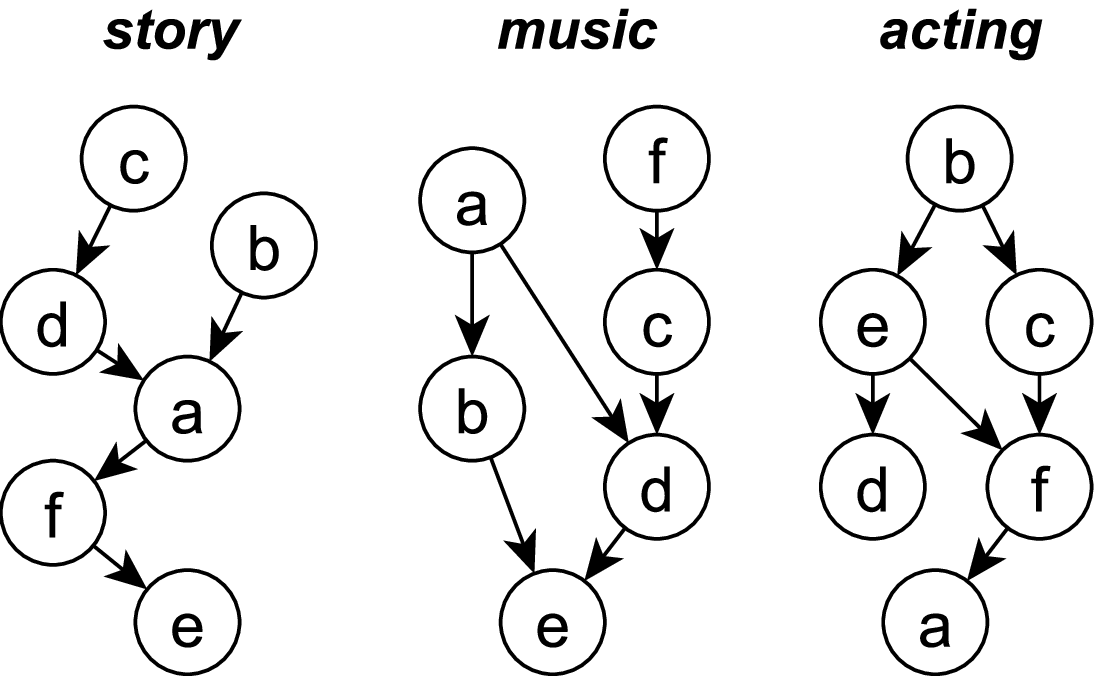} \vspace{-1mm}
    \caption{\small{Preference relations (i.e., strict partial orders) on three criteria.}}
    \vspace{2mm}
    \label{fig:poset-movie}
    \end{subfigure}

    \begin{subfigure}[b]{\linewidth}
        \begin{center}
        \begin{scriptsize}
        \begin{tabular}{|c|c|c|c|c|}
        \hline
        & \multicolumn{3}{c|}{ANSWER} & \\ \cline{2-4}
        QUESTION & $\succ$ &$\sim$&$\prec$ & OUTCOME \\ \hline
        \paircomp{a}{b}{s}&1&0&4&\domC{b}{a}{s}\\ \hline
        \paircomp{a}{c}{s}&0&0&5&\domC{c}{a}{s}\\ \hline
        \paircomp{a}{d}{s}&0&2&3&\domC{d}{a}{s}\\ \hline
        \paircomp{a}{e}{s}&4&0&1&\domC{a}{e}{s}\\ \hline
        \paircomp{a}{f}{s}&3&1&1&\domC{a}{f}{s}\\ \hline
        \paircomp{b}{c}{s}&1&2&2&\indiffC{b}{c}{s}\\ \hline
        \paircomp{b}{d}{s}&1&3&1&\indiffC{b}{d}{s}\\ \hline
        \paircomp{b}{e}{s}&5&0&0&\domC{b}{e}{s}\\ \hline
        \paircomp{b}{f}{s}&4&1&0&\domC{b}{f}{s}\\ \hline
        \paircomp{c}{d}{s}&3&2&0&\domC{c}{d}{s}\\ \hline
        \paircomp{c}{e}{s}&4&0&1&\domC{c}{e}{s}\\ \hline
        \paircomp{c}{f}{s}&3&1&1&\domC{c}{f}{s}\\ \hline
        \paircomp{d}{e}{s}&3&0&2&\domC{d}{e}{s}\\ \hline
        \paircomp{d}{f}{s}&3&2&0&\domC{d}{f}{s}\\ \hline
        \paircomp{e}{f}{s}&1&1&3&\domC{f}{e}{s}\\ \hline
        \end{tabular}
        \end{scriptsize}
        \end{center}
    \vspace{-3mm}
    \caption{\small{Deriving the preference relation for criterion \emph{story} by pairwise comparisons. Each comparison is performed by $5$ workers. $\theta = 60\%$.}}
    \label{fig:crowd-pref-movie}
    \end{subfigure}
\vspace{-4mm}
\caption{\small{Finding Pareto-optimal movies by \emph{story}, \emph{music} and \emph{acting}.}}
\label{fig:movie-example} 
\end{figure}

Pareto-optimal object finding lends itself to applications in several areas,
including public opinion collection, group decision making, and information
exploration, exemplified by the following motivating examples.

\begin{example}[Collecting Public Opinion and Group Decision Making]
\label{ex:movie}

Consider a set of movies $O$$=$$\{\obj{a,b,c,d,e,f}\}$ and a set of criteria
$C$$=$$\{$\emph{story, music, acting}$\}$ (denoted by $s$, $m$, $a$ in the
ensuing discussion).
Fig.\ref{fig:poset-movie} shows the individual preference relations
(i.e., strict partial orders), one per criterion.  Each strict partial
order is graphically represented as a directed acyclic graph (DAG),
more specifically a Hasse diagram.
The existence of a simple path from \obj{x} to \obj{y} in the DAG means
\obj{x} is better than (preferred to) \obj{y} by the corresponding criterion.
For example, $(\obj{a},\obj{e})$$\in$$P_m$
(\domC{a}{e}{m}), i.e., \obj{a} is better than \obj{e} by \emph{music}.
$(\obj{b},\obj{d})$$\notin$$P_s$ and $(\obj{d},\obj{b})$$\notin$$P_s$; hence \indiffC{b}{d}{s}.
The partial orders define the dominance relation between objects.
For instance, movie \obj{c} dominates \obj{d} (\domO{c}{d}), because \obj{c} is preferred
than \obj{d} on \emph{story} and \emph{music} and they are indifferent on
\emph{acting}, i.e., \domC{c}{d}{s}, \domC{c}{d}{m}, and \indiffC{c}{d}{a};
\obj{a} and \obj{b} do not dominate each other (\indiffO{a}{b}),
since \domC{b}{a}{s}, \domC{a}{b}{m} and \domC{b}{a}{a}.
Based on the three partial orders, movie \obj{b} is the only Pareto-optimal
object, since no other objects dominate it and every other object is dominated
by some object.

Note that tasks such as the above one may be used in both understanding the public's
preference (i.e., the preference relations are collected from a large,
anonymous crowd) and making decisions for a target group (i.e., the preference
relations are from a small group of people). \qed
\end{example}

\begin{example}[Information Exploration]
\label{ex:pic}

Consider a photography enthusiast, \emph{Amy}, who is drown in a large
number of photos she has taken and wants to select a subset of the better ones.
She resorts to crowdsourcing for the task, as it has been
exploited by many for similar tasks such as photo tagging, location/face
identification, sorting photos by (guessed) date, and so on.
Particularly, she would like to choose Pareto-optimal photos with regard to \emph{color},
\emph{sharpness} and \emph{landscape}. \qed \vspace{-1mm}
\end{example}

By definition, the crux of finding Pareto-optimal objects lies in obtaining the
preference relations, i.e., the strict partial orders on individual criteria.
Through crowdsourcing, the preference relations are derived by aggregating
the crowd's responses to \emph{pairwise comparison} tasks.  Each such
comparison between objects \obj{x} and \obj{y} by criterion $c$ is a question,
denoted \compxyc, which has three possible outcomes---\domC{x}{y}{c}, \domC{y}{x}{c}, and \indiffC{x}{y}{c},
based on the crowd's answers.
An example is as follows.\vspace{-2mm}

\begin{figure}
    \centering
    \vspace{-1mm}
    \includegraphics[width=75mm]{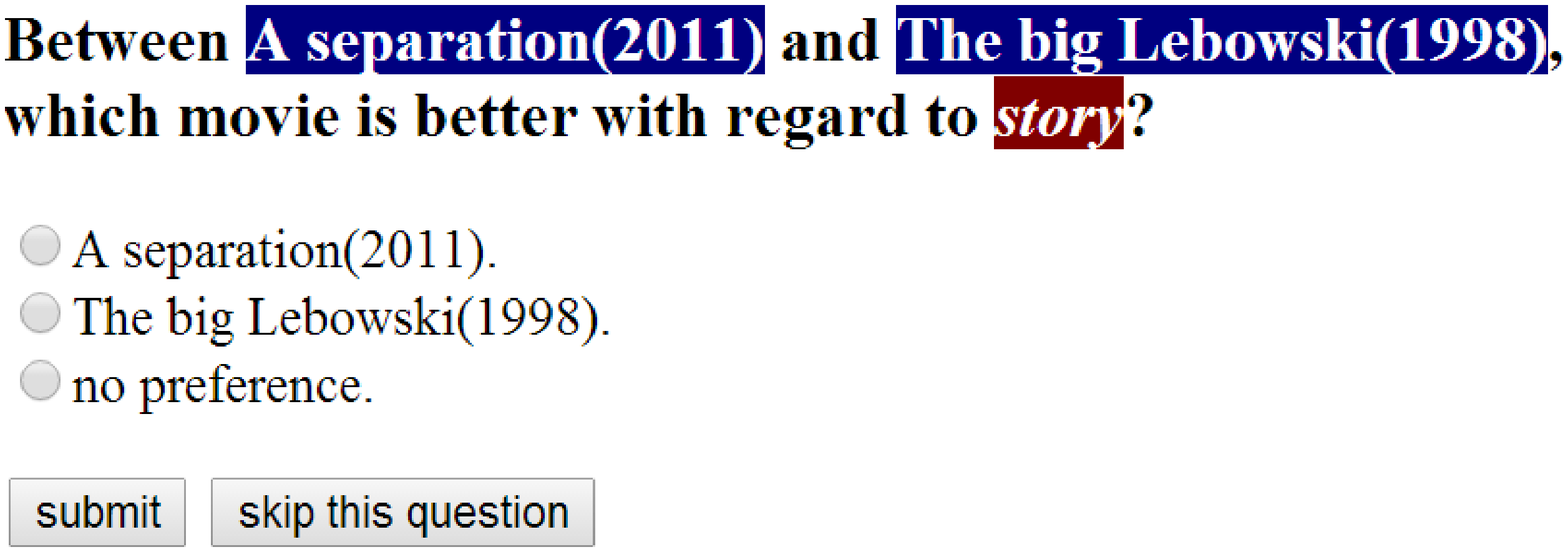}\vspace{-1mm}
    \caption{\small{A question that asks to compare two movies by \emph{story}.}}\vspace{-1mm}
    \label{fig:movieform}
\end{figure}

\begin{example}[Deriving Preference Relations from Pairwise Comparisons by the Crowd]
\label{ex:pairwise}
Fig.\ref{fig:crowd-pref-movie} shows the hypothetical results of all $15$
pairwise comparisons between the $6$ movies in Example~\ref{ex:movie}, by
criterion $s$$=$\emph{story}.
The outcomes of all comparisons form the crowd's preference relation on
\emph{story} (the leftmost DAG in Fig.\ref{fig:poset-movie}).
Fig.\ref{fig:movieform} is the screenshot of a question form designed
for one such comparison.
A crowdsourcer, when facing this question, would make a choice among
the three possible answers or skip a question if they do not have enough
confidence or knowledge to answer it.
Fig.\ref{fig:crowd-pref-movie} shows how many crowdsourcers
have selected each answer.  For instance, for question \paircomp{a}{f}{s},
three people preferred movie \obj{a}, one person preferred $f$, and one person
is indifferent.  By aggregating these answers, it is derived that \obj{a} is
better than \obj{f} with regard to \emph{story}, since $60\%$
of the crowdsourcers who responded to the question chose this answer.
For question \paircomp{b}{c}{s}, the result is \indiffC{b}{c}{s}, since neither
\domC{b}{c}{s} nor \worseC{b}{c}{s} received enough votes.
(Assuming a threshold $\theta$$=$$60\%$, i.e., either \domC{b}{c}{s} or
\worseC{b}{c}{s} should have at least $60\%$ of votes, in order to not
declare \indiffC{b}{c}{s}.)
\qed  \vspace{-1mm}
\end{example}

To the best of our knowledge, this paper is the first work on crowdsourcing
Pareto-optimal object finding.  The definition of Pareto-optimal objects follows
the concept of \emph{Pareto composition} of preference relations in~\cite{chomicki2003preference}.
It also resembles the definition of \emph{skyline} objects on totally-ordered attribute
domains (pioneered by~\cite{borzsony2001skyline}) and partially-ordered domains~\cite{chan2005stratified,sacharidis2009topologically,sarkas2008categorical,skylinepartialorder-zhang-vldb10}.
However, except for~\cite{lofi2013skyline}, previous studies on preference and skyline
queries do not use the crowd; they focus on query processing on \emph{existing} data.
On the contrary, we examine how to ask the crowd as few questions as
possible in obtaining sufficient data for determining Pareto-optimal
objects.  Furthermore, our work differs from preference and skyline queries
(including~\cite{lofi2013skyline}) in several radical ways:

\begin{list}{$\bullet$}
{ \setlength{\leftmargin}{0.5em} \setlength{\itemsep}{0pt} \setlength{\parsep}{0pt}}
\item The preference relation for a criterion is \emph{not} governed by
explicit scores or values on object attributes (e.g., sizes of houses, prices
of hotels), while preference and skyline queries on both totally- and partially-ordered domains
assumed explicit attribute representation.  For many comparison criteria, it is difficult to
model objects by explicit attributes, not to mention asking people to provide such values or
scores; people's preferences are rather based on complex, subtle personal perceptions, as
demonstrated in Examples~\ref{ex:movie} and~\ref{ex:pic}.

\item Due to the above reason, we request crowdsourcers to perform pairwise
comparisons instead of directly providing attribute values or scores.
On the contrary, \cite{lofi2013skyline} uses the crowd to obtain missing
attribute values.  Pairwise comparison is extensively studied in
social choice and welfare, preferences, and voting.
It is known that people are more comfortable and confident with
comparing objects than directly scoring them, since it
is easier, faster, and less error-prone~\cite{thurstone1927law}.

\item The crowd's preference relations are modeled as strict partial orders,
as opposed to bucket orders or full orders.
This is not only a direct effect of using pairwise comparisons instead of
numeric scores or explicit attribute values, but also a reflection of
the psychological nature of human's preferences~\cite{kiessling2002foundations,chomicki2003preference},
since it is not always natural to enforce a total or bucket order.
Most studies on skyline queries assume total/bucket orders,
except for~\cite{chan2005stratified,sacharidis2009topologically,sarkas2008categorical,skylinepartialorder-zhang-vldb10} which consider partial orders.
\end{list}

Our objective is to find all Pareto-optimal objects with as few questions
as possible.  A brute-force approach
will obtain the complete preference relations via pairwise
comparisons on all object pairs by every criterion.  However, without
such exhaustive comparisons, the incomplete knowledge collected from a
small set of questions may suffice in discerning all Pareto-optimal
objects.  Toward this end, it may appear that we can take advantage of
the transitivity of object dominance---a cost-saving property often
exploited in skyline query algorithms (e.g., \cite{borzsony2001skyline})
to exclude dominated objects from participating in any future comparison once
they are detected.
But, we shall prove that object dominance in our case is \emph{not} transitive
(Property~\ref{property:intransitivity}), due to the lack of explicit attribute representation.
Hence, the aforementioned cost-saving technique is inapplicable.

Aiming at Pareto-optimal object finding by a short sequence of questions,
we introduce a general, iterative algorithm
framework (Sec.\ref{sec:framework}).
Each iteration goes through four steps---\emph{question
selection}, \emph{outcome derivation}, \emph{contradiction resolution},
and \emph{termination test}.  In the $i$-th iteration, a question
$q_i$$=$$\compxyc$ is selected and its outcome is determined based
on crowdsourcers' answers.
On unusual occasions, if the outcome presents a contradiction to the obtained
outcomes of other questions, it is changed to the closest outcome such that the
contradiction is resolved.  Based on the transitive closure of the outcomes to
the questions so far, the objects $O$ are partitioned into three
sets---$O_\surd$ (objects that must be Pareto-optimal),
$O_\times$ (objects that must be non-Pareto optimal),
and $O_?$ (objects whose Pareto-optimality cannot be fully discerned by the
incomplete knowledge so far).
When $O_?$ becomes empty, $O_\surd$ contains all Pareto-optimal objects and the
algorithm terminates.  The question sequence so far is thus a \emph{terminal sequence}.

\begin{table*}
\begin{center}
\small
\begin{tabular}{|c||p{3.8cm}|p{3cm}|p{1.5cm}|p{3cm}|p{2.4cm}|}
\hline
 &\textbf{Task} & \textbf{Question type} & \textbf{Multiple attributes} & \textbf{Order among objects (on each attribute)} & \textbf{Explicit attribute representation} \\ \hline
 \cite{pairwiseranking-chen-wism13} & full ranking & pairwise comparison & no & bucket/total order & no \\ \hline
 \cite{humantopk} & \topk\ ranking & rank subsets of objects & no & bucket/total order & no \\ \hline
 \cite{davidson2013topkgroupby} & \topk\ ranking and grouping & pairwise comparison & no & bucket/total order & no \\ \hline
 \cite{lofi2013skyline} & skyline queries & missing value inquiry & yes & bucket/total order & yes \\ \hline
 This work & Pareto-optimal object finding & pairwise comparison & yes & strict partial order & no \\ \hline
\end{tabular}
\end{center}
\vspace{-2mm}
\caption{\small{Related work comparison.}}
\label{table:related}
\vspace{-5mm}
\end{table*}

There are a vast number of terminal sequences.
Our goal is to find one that is as short as possible.
We observe that, for a non-Pareto optimal object,
knowing that it is dominated by at least one object is sufficient, and we
do not need to find all its dominating objects.
It follows that we do not really care about
the dominance relation between non-Pareto optimal objects and we can
skip their comparisons.
Hence, the overriding principle of our question selection strategy is to
identify non-Pareto optimal objects as early as possible.
Guided by this principle, the framework only chooses from
\emph{candidate questions} which must satisfy three
conditions (Sec.\ref{sec:q-select}).
This design is sufficient, as we prove that an empty candidate
question set implies a terminal sequence, and vice versa (Proporty~\ref{property:emptycand}).  The design is also efficient,
as we further prove that, if a question sequence
contains non-candidate questions, there exists a shorter or equally long
sequence with only candidate questions that produces the same $O_\times$,
matching the principle of eagerly finding non-Pareto optimal objects (Theorem~\ref{thm:optimal}).
Moreover, by the aforementioned principle, the framework selects in every iteration
such a candidate question \compxyc\ that \obj{x} is more likely to be dominated by \obj{y}.
The selection is steered by two ideas---\emph{macro-ordering}
and \emph{micro-ordering}.
By using different micro-ordering heuristics,  the framework is instantiated into several algorithms with varying power in
pruning questions (Sec.\ref{sec:alg}).
We also derive a lower bound on the number of questions required for finding all
Pareto-optimal objects (Theorem~\ref{th:lowerbound}).

In summary, this paper makes the following contributions:
\begin{list}{$\bullet$}
{ \setlength{\leftmargin}{0.5em} \setlength{\parsep}{0pt}}
  \item This is the first work on crowdsourcing Pareto-optimal object finding.
  Prior studies on crowdsourcing skyline queries~\cite{lofi2013skyline} assumes explicit attribute representation and uses crowd to obtain missing attribute values.
  We define preference relations purely based on pairwise comparisons and we aim to find all Pareto-optimal objects by as few comparisons as possible.
  \item We propose a general, iterative algorithm framework (Sec.\ref{sec:framework}) which follows the strategy of choosing only candidate questions that must satisfy three conditions.
  We prove important properties that establish the advantage of the strategy (Sec.\ref{sec:q-select}).
  \item We design macro-ordering and micro-ordering heuristics for finding a short terminal question sequence.  Based on the heuristics, the generic framework is instantiated into several algorithms (\algname{RandomQ}, \algname{RandomP}, \algname{FRQ}) with varying efficiency.  We also derive a non-trivial lower bound on the number of required pairwise comparison questions.  (Sec.\ref{sec:alg})
  \item  We carried out experiments by simulations to compare the amount of comparisons required by different instantiations of the framework under varying problem sizes.  We also investigated two case studies by using human judges and real crowdsourcing marketplace.  The results demonstrate the effectiveness of selecting only candidate questions, macro-ordering, and micro-ordering.  When these ideas are stacked together, they use orders of magnitude less comparisons than a brute-force approach.  The results also reveal that \algname{FRQ} is nearly optimal and the lower bound is practically tight, since \algname{FRQ} gets very close to the lower bound. (Sec.\ref{sec:exp})
\end{list}

\section{Related Work} \label{s:related}


This is the first work on crowdsourcing Pareto-optimal object finding.
There are several recent studies on using crowdsourcing to rank objects and answer group-by, \topk\ and skyline queries. \algname{Crowd-BT}~\cite{pairwiseranking-chen-wism13} ranks objects by crowdsourcing pairwise object comparisons.  Polychronopoulos et al.~\cite{humantopk} find \topk\ items in an itemset by asking human workers to rank small subsets of items.  Davidson et al.~\cite{davidson2013topkgroupby} evaluate \topk\ and group-by queries by asking the crowd to answer \emph{type} questions (whether two objects belong to the same group) and \emph{value} questions (ordering two objects).  Lofi et al.~\cite{lofi2013skyline} answer skyline queries over incomplete data by asking the crowd to provide missing attribute values.
Table~\ref{table:related} summarizes the similarities and differences between these studies and our work.  The studies on full and \topk\ ranking~\cite{pairwiseranking-chen-wism13,humantopk,davidson2013topkgroupby} do not consider multiple attributes in modeling objects.  On the contrary, the concepts of skyline~\cite{lofi2013skyline} and Pareto-optimal objects (this paper) are defined in a space of multiple attributes.  \cite{lofi2013skyline} assumes explicit attribute representation.  Therefore, they resort to the crowd for completing missing values, while other studies including our work request the crowd to compare objects.  Our work considers strict partial orders among objects on individual attributes.  Differently, other studies assume a bucket/total order~\cite{pairwiseranking-chen-wism13,humantopk,davidson2013topkgroupby} or multiple bucket/total orders on individual attributes~\cite{lofi2013skyline}.

Besides~\cite{pairwiseranking-chen-wism13}, there were multiple studies on ranking objects by pairwise comparisons,
which date back to decades ago as aggregating the preferences of multiple agents has always been a fundamental problem in social choice and welfare~\cite{arrow1951social}.
The more recent studies can be categorized into three types: \textbf{1)} Approaches such as~\cite{liu-cikm09,rendle-uai09,crowdranking-yi-hcomp13} predict users' object ranking by completing a user-object scoring matrix.  Their predications take into account users' similarities in pairwise comparisons, resembling \emph{collaborative filtering}~\cite{collaborativefiltering-goldberg-cacm92}.  They thus do not consider explicit attribute representation for objects.  \textbf{2)} Approaches such as~\cite{ranknet-burges-icml05,irsvm-cao-sigir06,lambdarank-burges-nips06} infer query-specific (instead of user-specific) ranked results to web search queries.  Following the paradigm of \emph{learning-to-rank}~\cite{learntorank-liu-09}, they rank a query's result documents according to pairwise result comparisons of other queries.  The documents are modeled by explicit ranking features.   \textbf{3)} Approaches such as~\cite{braverman-soda08,jamieson-nips11,ailon-nips11,ailon-jmlr12,negahban-nips12} are similar to~\cite{pairwiseranking-chen-wism13} as they use pairwise comparisons to infer a single ranked list that is neither user-specific nor query-specific.  Among them, \cite{jamieson-nips11} is special in that it also applies learning-to-rank and requires explicit feature representation.  Different from our work, none of these studies is about Pareto-optimal objects, since they all assume a bucket/total order among objects; those using learning-to-rank require explicit feature representation, while the rest do not consider multiple attributes.
Moreover, except~\cite{jamieson-nips11,ailon-nips11,ailon-jmlr12}, they all assume comparison results are already obtained before their algorithms kick in.
In contrast, we aim at minimizing the pairwise comparison questions to ask in finding Pareto-optimal objects.

\section{General Framework} \label{sec:framework}

By the definition of Pareto-optimal objects, the key to finding such
objects is to obtain the preference relations, i.e., the strict
partial orders on individual criteria.  Toward this end, the most basic
operation is to perform \emph{pairwise comparison}---given a pair of
objects \obj{x} and \obj{y} and a criterion $c$, determine whether one is better
than the other (i.e., $(\obj{x},\obj{y}) \in P_c$ or $(\obj{y},\obj{x}) \in P_c$) or they are
indifferent (i.e., $(\obj{x},\obj{y}) \notin P_c \wedge (\obj{y},\obj{x}) \notin P_c$).

The problem of crowdsourcing Pareto-optimal object finding is thus essentially
crowdsourcing pairwise comparisons.  Each comparison task between
\obj{x} and \obj{y} by criterion $c$ is presented to the crowd as a
question $q$ (denoted \compxyc).  The outcome to the question
(denoted $\rlt(q)$) is aggregated from the crowd's answers.
Given a set of questions, the outcomes thus contain an (incomplete) knowledge
of the crowd's preference relations for various criteria.
Fig.\ref{fig:movieform} illustrates the screenshot of one such question
(comparing two movies by \emph{story}) used in our empirical evaluation.
We note that there are other viable designs of question, e.g., only
allowing the first two choices (\domC{x}{y}{c} and \domC{y}{x}{c}).
Our work is agnostic to the specific question design.

Given $n$ objects and $r$ criteria, a brute-force approach will perform
pairwise comparisons on all object pairs by every criterion, which
leads to $r$$\cdot$$n$$\cdot$$(n$$-$$1)/2$ comparisons.  The
corresponding question outcomes amount to the complete underlying preference
relations.  The quadratic nature
of the brute-force approach renders it wasteful.  The bad news is that, in the worst
case, we cannot do better than it.  To understand
this, consider the scenario where all objects are indifferent by
every criterion.  If any comparison \compxyc\ is skipped, we cannot
determine if \obj{x} and \obj{y} are indifferent or if one dominates another.

In practice, though, the outlook is much brighter.  Since we look for
only Pareto-optimal objects, it is an overkill to obtain the complete
preference relations.  Specifically, for a Pareto-optimal object, knowing
that it is not dominated by any object is sufficient, and we do not need to
find all the objects dominated by it; for a non-Pareto optimal object,
knowing that it is dominated by at least one object is sufficient, and we
do not need to find all its dominating objects.
Hence, without exhausting all possible comparisons, the incomplete knowledge
on preference relations collected from a set of questions may suffice in
fully discerning all Pareto-optimal objects.

Our objective is to find all Pareto-optimal objects with as few questions as
possible.  By pursuing this goal, we are applying a very simple
cost model---the cost of a solution only depends on its number of questions.
Although the cost of a task in a crowdsourcing environment
may depend on monetary cost, latency and other factors, the number of questions
is a generic, platform-independent cost measure and arguably proportionally
correlates with the real cost.  Therefore, we assume a sequential execution model which
asks the crowd an ordered sequence of questions
$Q=\langle q_1,...,q_n \rangle$---it only asks $q_{i+1}$ after $\rlt(q_i)$
is obtained.  Thereby, we do not consider asking multiple questions concurrently.
Furthermore, in discussion of our approach, the focus shall be on how to find a
short question sequence instead of the algorithms' complexity.

\begin{figure}
    \centering
    \includegraphics[width=85mm]{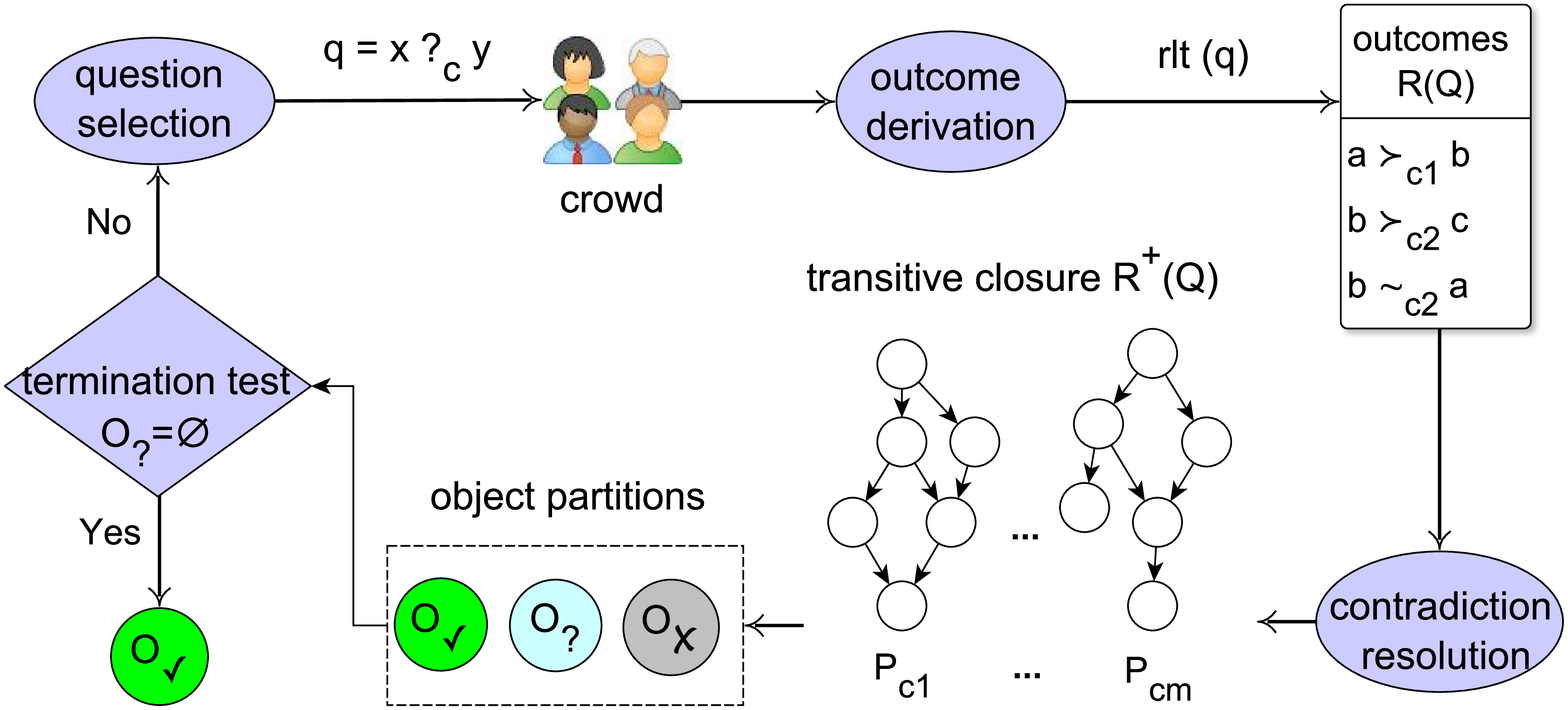}
    \caption{\small{The general framework.}}
    \label{fig:framework}
    \vspace{-2mm}
\end{figure}

\begin{algorithm}[t]
\small
\LinesNumbered
\SetCommentSty{textsf}

\KwIn{$O$ {\textsf{\footnotesize: the set of objects}}}

\KwOut{$O_\surd$ {\footnotesize \textsf{: Pareto-optimal objects of $O$}}}

\BlankLine
$R(Q) \leftarrow \emptyset$\tcc*{\footnotesize question outcomes}

\Repeat{$O_?=\{\}$}
{
	$\compxyc\ \leftarrow$ question selection\;
	$rlt(\compxyc) \leftarrow$ outcome derivation\tcc*{\footnotesize resolve conflict, if any}
	$R(Q) \leftarrow R(Q)\cup \{ rlt(\compxyc) \}$\;
	$(O_\surd, O_\times, O_?) \leftarrow$ partitioning objects based on $R^+(Q)$\tcc*{\footnotesize $R^+(Q)$ is the transitive closure of $R(Q)$}
}
\Return $O_\surd$\;
\caption{The general framework}
\label{alg:general}
\end{algorithm}

To find a short sequence, we design a general algorithm
framework, as displayed in Fig.\ref{fig:framework}.
Alg.\ref{alg:general} shows the
framework's pseudo-code.  Its execution is iterative.  Each
iteration goes through four steps---\emph{question selection},
\emph{outcome derivation}, \emph{contradiction resolution},
and \emph{termination test}.  In the $i$-th iteration, a question
$q_i$$=$$\compxyc$ is selected and presented to the crowd.  The question outcome
$rlt(q_i)$ is derived from the crowd's aggregated answers.
On unusual occasions, if the outcome presents a contradiction to the obtained
outcomes of other questions so far, it is changed to the closest outcome to resolve
contradiction.  By computing $R^+(Q_i)$, the \emph{transitive closure} of $R(Q_i)$---the
obtained outcomes to questions so far $\langle q_1, \ldots, q_{i} \rangle$, the outcomes
to certain questions are derived and such questions will never be asked.  Based on $R^+(Q_i)$,
if every object is determined to be either Pareto-optimal or non-Pareto
optimal without uncertainty, the algorithm terminates.

Below, we discuss outcome derivation and termination test.
Sec.\ref{sec:q-select} examines the framework's key step---question selection,
and Sec.\ref{sec:contradict} discusses contradiction resolution.

\vspace{-2mm}
{\flushleft \textbf{Outcome derivation}}\hspace{2mm}
Given a question \compxyc, its outcome \rlt(\compxyc) must be aggregated
from multiple crowdsourcers, in order to reach a reliable result with confidence.
Particularly, one of three mutually-exclusive outcomes is determined based
on $k$ crowdsourcers' answers to the question:\vspace{-2mm}
\begin{equation}\label{eq:outcome}
\hspace{-1mm}\rlt(\compxyc)=\left\{
\begin{array}{lll}
	\domC{x}{y}{c} & \text{if}\ \frac{\#\obj{x}}{k} \geq \theta \\
	\domC{y}{x}{c} & \text{if}\ \frac{\#\obj{y}}{k} \geq \theta \\
	\indiffC{x}{y}{c}\ (\obj{x} \nsucc_c \obj{y} \wedge \obj{y} \nsucc_c \obj{x}) & \text{otherwise}
\end{array}
\right.
\end{equation}
where $\theta$ is such a predefined threshold that $\theta$$>$$50\%$,
$\#\obj{x}$ is the number of crowdsourcers (out of $k$) preferring \obj{x} over \obj{y}
on criterion $c$, and $\#\obj{y}$ is the number of crowdsourcers preferring \obj{y}
over \obj{x} on $c$.  Fig.\ref{fig:crowd-pref-movie} shows the outcomes of all
$15$ questions according to Equation~(\ref{eq:outcome}) for comparing
movies by \emph{story} using $k$$=$$5$ and $\theta$$=$$60\%$.
Other conceivable definitions may be used in determining
the outcome of \compxyc.  For example, the outcome may be defined as
the choice (out of the three possible choices) that receives the most
votes from the crowd.  The ensuing discussion is agnostic to the specific
definition.

The current framework does not consider different levels of confidence on question outcomes.  The confidence on the outcome of a question may be represented as a probability value based on the distribution of crowdsourcers' responses. An interesting direction for future work is to find Pareto-optimal objects in probabilistic sense.  The confidence may also reflect the crowdsourcers' quality and credibility~\cite{quality-ipeirotis-2010}.

\vspace{-2mm}
{\flushleft \textbf{Termination test}}\hspace{2mm}
In each iteration, Alg.\ref{alg:general} partitions the objects into
three sets by their Pareto-optimality based on the transitive closure
of question outcomes so far.  If every object's
Pareto-optimality has been determined without uncertainty, the algorithm
terminates.  Details are as follows. \vspace{-1mm}


\begin{definition}[Transitive Closure of Outcomes] \label{def:closure}
Given a set of questions $Q$$=$$\langle q_1,...,q_n \rangle$, the transitive
closure of their outcomes $R(Q)$$=$ $\{\rlt(q_1),$ $...,\rlt(q_n)\}$ is
$R^+(Q)$$=$\{\indiffC{x}{y}{c} $|$ \indiffC{x}{y}{c} $\in R(Q)$\} $\bigcup$
\{\domC{x}{y}{c} $|\; ($\domC{x}{y}{c} $\in R(Q)) \vee ( \exists$ \obj{w$_1$,w$_2$,...,w$_m$} : \obj{w$_1$}$=$\obj{x}, \obj{w$_m$}$=$\obj{y} $\wedge \; (
\forall$$0$$<$$i$$<$$m$ : \obj{w$_i$} $\succ_c$ \obj{w$_{i+1}$} $\in R(Q)))$ \}. \qed \vspace{-2mm}
\end{definition}

In essence, the transitive closure dictates \domC{x}{z}{c} without asking
the question \paircomp{x}{z}{c}, if the existing outcomes $R(Q)$ (and recursively
the transitive closure $R^+(Q)$) contains both \domC{x}{y}{c} and
\domC{y}{z}{c}.  Based on $R^+(Q)$, the objects $O$ can be partitioned into three
sets: \vspace{-2mm}

{\flushleft $O_\surd = \{\obj{x}\in O\ |\ \forall \obj{y}\in O : (\exists c\in C :$ \domC{x}{y}{c} $\in R^+(Q)) \vee  (\forall c\in C :$ \indiffC{x}{y}{c} $\in R^+(Q))\}$;}\vspace{-2mm}
{\flushleft $O_\times =\{\obj{x}\in O\; |\: \exists \obj{y}\in O : (\forall c\in C :$ \domC{y}{x}{c} $\in R^+(Q) \vee$ \indiffC{x}{y}{c} $\in R^+(Q)) \wedge (\exists c\in C :$ \domC{y}{x}{c} $\in R^+(Q))\}$;}\vspace{-2mm}
{\flushleft $O_? = O \backslash (O_\surd \cup O_\times)$.}\vspace{-2mm}

{\flushleft $O_\surd$ contains} objects that must be Pareto-optimal,
$O_\times$ contains objects that cannot possibly be Pareto-optimal,
and $O_?$ contains objects for which the incomplete knowledge $R^+(Q)$
is insufficient for discerning their Pareto-optimality.  The objects
in $O_?$ may turn out to be Pareto-optimal after more
comparison questions.  If the set $O_?$ for a question sequence $Q$
is empty, $O_\surd$ contains all Pareto-optimal objects and the
algorithm terminates.  We call such a $Q$ a \emph{terminal sequence},
defined below. \vspace{-1mm}

\begin{definition}[Terminal Sequence]
A question sequence $Q$ is a terminal sequence if and only if,
based on $R^+(Q)$, $O_?$$=$$\emptyset$.
\end{definition}

\subsection{Question Selection}\label{sec:q-select}

Given objects $O$ and criteria $C$, there can be a huge number of terminal
sequences.  Our goal is to find a sequence as short as possible.
As Fig.\ref{fig:framework} and Alg.\ref{alg:general} show, the framework
is an iterative procedure of object partitioning based on question outcomes.
It can also be viewed as the process of moving objects from $O_?$ to
$O_\surd$ and $O_\times$.  Once an object is moved to $O_\surd$ or $O_\times$,
it cannot be moved again.  With regard to this process, we make two important
observations, as follows.

\begin{list}{$\bullet$}
{ \setlength{\leftmargin}{0.5em} \setlength{\itemsep}{0pt} }

\item \emph{In order to declare an object \obj{x} not Pareto-optimal, it is sufficient
to just know \obj{x} is dominated by another object}.
It immediately follows that we do not really care about
the dominance relationship between objects in $O_\times$ and thus can
skip the comparisons between such objects. Once we know
\obj{x}$\in$$O_?$ is dominated by another object, it cannot be
Pareto-optimal and is immediately moved to $O_\times$.
Quickly moving objects into $O_\times$ can allow us
skipping many comparisons between objects in $O_\times$.

\item \emph{In order to declare an object \obj{x} Pareto-optimal, it is necessary to know
that no object can dominate \obj{x}.} This means we may need to compare
\obj{x} with all other objects including non Pareto-optimal objects.  As an extreme example,
it is possible for \obj{x} to be dominated by only a non-Pareto optimal object
\obj{y} but not by any other object (not even the objects dominating \obj{y}).  This is because
object dominance based on preference relations is intransitive, which is formally
stated in Property~\ref{property:intransitivity}.

\begin{property}[Intransitivity of Object Dominance]\label{property:intransitivity}
Object dominance based on the preference relations over a set of criteria is
not transitive.  Specifically, if \domO{x}{y} and \domO{y}{z}, it is not
necessarily true that \domO{x}{z}.  In other words, it is possible that
\indiffO{x}{z} or even \domO{z}{x}. \qed
\end{property}

We show the intransitivity of object dominance by an example. Consider
objects $O$$=$$\{\obj{x,y,z}\}$, criteria $C$$=$$\{c_1,c_2,c_3\}$, and the preference
relations in Fig.\ref{fig:intransitivity}.  Three dominance
relationships violate transitivity: (i) \domO{x}{y} (based on
\domC{x}{y}{c_1}, \indiffC{x}{y}{c_2}, \indiffC{x}{y}{c_3}), (ii) \domO{y}{z}
(based on \indiffC{y}{z}{c_1}, \domC{y}{z}{c_2}, \indiffC{y}{z}{c_3}), and (iii)
\domO{z}{x} (based on \indiffC{z}{x}{c_1}, \indiffC{z}{x}{c_2}, \domC{z}{x}{c_3}).
As another example, in Fig.\ref{fig:poset-movie}, \domO{b}{c} (since
\indiffC{b}{c}{s}, \indiffC{b}{c}{m}, \domC{b}{c}{a}, where $s$$=$\emph{story}, $m$$=$\emph{music},
$a$$=$\emph{acting}) and \domO{c}{a} (since \domC{c}{a}{s}, \indiffC{c}{a}{m},
\domC{c}{a}{a}), but \indiffO{a}{b} (since \domC{b}{a}{s}, \domC{a}{b}{m},
\domC{b}{a}{a}) where transitivity does not hold.

\begin{figure}
    \centering
    \includegraphics[width=65mm]{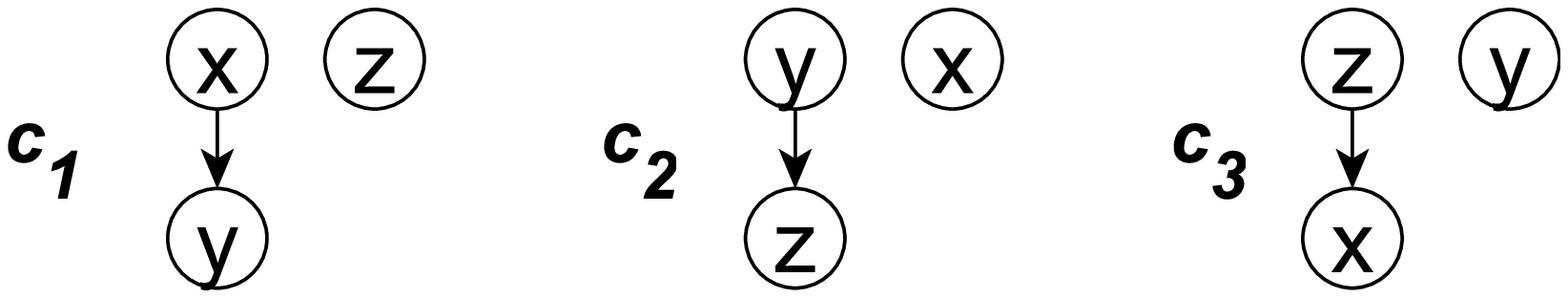}\vspace{-1mm}
    \caption{\small{Intransitivity of object dominance: \domO{x}{y}, \domO{y}{z}, \domO{z}{x}.}}
    \label{fig:intransitivity}
    \vspace{-1mm}
\end{figure}

Differently, transitivity of object dominance holds in skyline
analysis~\cite{borzsony2001skyline}.  The contradiction is due to the lack
of explicit attribute representation---in our case two objects are
considered equally good on a criterion if they are indifferent, while in
skyline analysis they are equally good regarding an attribute if they bear
identical values.  Skyline query algorithms exploit the transitivity of
object dominance to reduce execution cost, because an object can be
immediately excluded from further comparison once it is found dominated by
any other object.  However, due to Property~\ref{property:intransitivity},
we cannot leverage such pruning anymore.
\end{list}

\begin{algorithm}[t]
\small
\LinesNumbered
\SetCommentSty{textsf}

\KwIn{$R^+(Q_i), O_\surd(Q_i), O_?(Q_i), O_\times(Q_i)$}

\KwOut{$Q^1_{can}$ or $Q^2_{can}$}

\BlankLine
$Q_{can} \leftarrow \{ \compxyc\ |\ rlt(\compxyc) \notin R^+(Q_i) \wedge \obj{x} \in O_?(Q_i) \wedge (\nexists c' \in C : \obj{x} \succ_{c'} \obj{y} \in R^+(Q_i)) \}$;

$Q^1_{can} \leftarrow \{ \compxyc\ |\ \compxyc\ \in Q_{can}, \obj{y} \notin O_\times(Q_i) \}$;

$Q^2_{can} \leftarrow \{ \compxyc\ |\ \compxyc\ \in Q_{can}, \obj{y} \in O_\times(Q_i) \}$;

\BlankLine
\tcc{\footnotesize \emph{Macro-ordering: consider $Q^1_{can}$ before $Q^2_{can}$.}}
\uIf{$Q^1_{can} \neq \emptyset$}{
    \Return \emph{micro-ordering}($Q^1_{can}$);
}
\uElseIf{$Q^2_{can} \neq \emptyset$}{
    \Return \emph{micro-ordering}($Q^2_{can}$);
}

\caption{Question selection}
\label{alg:qselect}
\end{algorithm}

Based on these observations, the overriding principle of our question
selection strategy (shown in Alg.\ref{alg:qselect}) is to identify
non-Pareto optimal objects as early as possible.  At every iteration of the framework
(Alg.\ref{alg:general}), we choose to compare \obj{x} and \obj{y} by criterion
$c$ (i.e., ask question \compxyc) where \compxyc\ belongs to
\emph{candidate questions}.  Such candidate questions must satisfy three
conditions (Definition~\ref{def:candidate}).
There can be many candidate questions.  In choosing the next question,
by the aforementioned principle, we select such \compxyc\ that
\obj{x} is more likely to be dominated by \obj{y}.  More specifically,
we design two ordering heuristics---\emph{macro-ordering}
and \emph{micro-ordering}.
Given the three object partitions $O_\surd$, $O_\times$
and $O_?$, the macro-ordering idea is simply that we choose \obj{x} from $O_?$
(required by one of the conditions on candidate questions)
and \obj{y} from $O_\surd \cup O_?$ (if possible) or $O_\times$ (otherwise).
The reason is that it is less likely for an object in $O_\times$ to dominate \obj{x}.
Micro-ordering further orders all candidate questions satisfying the macro-ordering heuristic.
In Sec.\ref{sec:alg}, we instantiate the framework into a variety of solutions with varying
power in pruning questions, by using different micro-ordering heuristics.

\begin{definition}[Candidate Question]\label{def:candidate}
Given $Q$, the set of asked questions so far, \compxyc\ is a
\emph{candidate question} if and only if it satisfies the following conditions:\vspace{-1mm}
\begin{enumerate}[(i)]
\itemsep=-1pt
\topsep=-1pt
  \item\label{item:1} The outcome of \compxyc\ is unknown yet, i.e., $rlt(\compxyc) \notin R^+(Q)$;
  \item\label{item:2} \obj{x} must belong to $O_?$;
  \item\label{item:3} Based on $R^+(Q)$, the possibility of \domO{y}{x} must not be ruled out yet, i.e., $\nexists c' \in C : \obj{x} \succ_{c'} \obj{y} \in R^+(Q)$.
\end{enumerate}
\vspace{-1mm}

We denote the set of candidate questions by $Q_{can}$.  Thus,
$Q_{can}$ $=$ $\{ \compxyc\ |\ rlt(\compxyc) \notin R^+(Q) \wedge \obj{x} \in O_? \wedge (\nexists c' \in C : \obj{x} \succ_{c'} \obj{y} \in R^+(Q)) \}$. \qed \vspace{-2mm}
\end{definition}

If no candidate question exists, the question sequence $Q$ is a terminal sequence.
The reverse statement is also true, i.e., upon a terminal sequence, there is no candidate question left.
This is formalized in the following property.\vspace{-1mm}

\begin{property}\label{property:emptycand}
$Q_{can} = \emptyset$ if and only if $O_? = \emptyset$.\vspace{-1mm}
\end{property}
\begin{proof}
It is straightforward that
$O_?$$=$$\emptyset$$\Rightarrow$$Q_{can}$$=$$\emptyset$, since an empty $O_?$
means no question can satisfy condition (\ref{item:2}).  We prove
$Q_{can}$$=$$\emptyset$$\Rightarrow$$O_?$$=$$\emptyset$ by proving its
equivalent contrapositive $O_?$$\neq$$\emptyset$$\Rightarrow$$Q_{can}$$\neq$$\emptyset$.
Assume $O_?$$\neq$$\varnothing$, i.e., $O_?$ contains at least one object \obj{x}
(condition (\ref{item:2}) satisfied).  Since \obj{x} does not belong to $O_\surd$, there
exists at least an object \obj{y} that may turn out to dominate \obj{x}
(condition (\ref{item:3}) satisfied).  Since \obj{x} is not in
$O_\times$, we cannot conclude yet that \obj{y} dominates
\obj{x}.  Hence, there must exist a criterion $c$ for which we do not know
the outcome of \obj{x}$?_c$\obj{y} yet, i.e., $rlt(\obj{x} ?_c
\obj{y})\notin R^+(Q)$ (condition (\ref{item:1}) satisfied).  The question
$\obj{x} ?_c \obj{y}$ would be a candidate question because it satisfies
all three conditions.  Hence $O_{can}$$\neq$$\varnothing$.
\end{proof}

Questions violating the three conditions may
also lead to terminal sequences. However, choosing only candidate questions matches
our objective of quickly identifying non-Pareto optimal objects.
Below we justify the conditions.

Condition (\ref{item:1}): This is straightforward.  If $R(Q)$ or its transitive
closure already contains the outcome of \compxyc, we do not ask the same question again.

Condition (\ref{item:2}): This condition essentially dictates that at least
one of the two objects in comparison is from $O_?$.  (If only one of them belongs
to $O_?$, we make it \obj{x}.)  Given a pair \obj{x} and \obj{y}, if neither is from $O_?$, there are three
scenarios---(1) $\obj{x}$$\in$$O_\surd, \obj{y}$$\in$$O_\surd$, (2) $\obj{x}$$\in$$O_\surd,
\obj{y}$$\in$$O_\times$ or $\obj{x}$$\in$$O_\times, \obj{y}$$\in$$O_\surd$, (3) $\obj{x}$$\in$$O_\times,
\obj{y}$$\in$$O_\times$.  Once we know an object is in $O_\surd$ or $O_\times$,
its membership in such a set will never change.  Hence, we are not interested
in knowing the dominance relationship between objects from $O_\surd$ and
$O_\times$ only.  In all these three scenarios, comparing \obj{x} and \obj{y} is only useful
for indirectly determining (by transitive closure) the outcome of comparing other objects.  Intuitively
speaking, such indirect pruning is not as efficient as direct pruning.

Condition (\ref{item:3}):  This condition requires that, when
\compxyc\ is chosen, we cannot rule out the possibility of \obj{y}
dominating \obj{x}.  Otherwise, if \obj{y} cannot possibly dominate \obj{x},
the outcome of \compxyc\ cannot help prune \obj{x}.  Note that, in
such a case, comparing \obj{x} and \obj{y} by $c$ may help prune
\obj{y}, if \obj{y} still belongs to $O_?$ and \obj{x} may
dominate \obj{y}.  Such possibility is not neglected and is
covered by a different representation of the same question---\paircomp{y}{x}{c}, i.e.,
swapping the positions of \obj{x} and \obj{y} in checking the three
conditions.  If it is determined \obj{x} and \obj{y} cannot dominate each
other, then their further comparison is only useful for indirectly
determining the outcome of comparing other objects.  Due to the same reason
explained for condition (\ref{item:2}), such indirect pruning is less
efficient.

The following simple Property~\ref{property:nondominance} helps to determine
whether \domO{y}{x} is possible: If \obj{x} is better than \obj{y} by any
criterion, then we can already rule out the possibility of \domO{y}{x},
without knowing the outcome of their comparison by every criterion.  This
allows us to skip further comparisons between them.  Its correctness is
straightforward based on the definition of object dominance.

\begin{property}[Non-Dominance Property] \label{property:nondominance}
At any given moment, suppose the set of asked questions is $Q$.
Consider two objects \obj{x} and \obj{y} for
which the comparison outcome is not known for every criterion, i.e.,
$\exists c$ such that $rlt(\compxyc) \notin R^+(Q)$.  It can be determined
that \obj{y}$\nsucc$\obj{x} if $\exists c\in C$ such that
\domC{x}{y}{c}$\in R^+(Q)$. \qed \vspace{-1mm}
\end{property}

In justifying the three conditions in defining candidate questions,
we intuitively explained that indirect pruning is less
efficient---if it is known that \obj{x} does not belong to $O_?$
or \obj{y} cannot possibly dominate \obj{x}, we will not ask question
\compxyc.  We now justify this strategy theoretically and precisely.
Consider a question sequence $Q$$=$$\langle q_1, \ldots, q_n \rangle$.
We use $O_\surd(Q)$, $O_?(Q)$, $O_\times(Q)$ to denote object partitions
according to $R^+(Q)$.
For any question $q_i$, the subsequence comprised of its preceding questions
is denoted $Q_{i-1}$$=$$\langle q_1, \ldots, q_{i-1} \rangle$.
If $q_i$ was not a candidate question when it was chosen (i.e.,
after $R(Q_{i-1})$ was obtained), we say it is a \emph{non-candidate}.
The following Theorem~\ref{thm:optimal} states that, if a question sequence
contains non-candidate questions, we can replace it by a shorter or equally long
sequence without non-candidate questions that produces the same set of
dominated objects $O_\times$.  Recall that the key to our framework is to recognize
dominated objects and move them into $O_\times$ as early as possible.  Hence,
the new sequence will likely lead to less cost when the algorithm terminates.
Hence, it is a good idea to only select among candidate questions.

\begin{theorem}\label{thm:optimal}
If $Q$ contains non-candidate questions, there exists a question sequence
$Q'$ without non-candidate questions such that $|Q'|\leq |Q|$ and $O_\times (Q') = O_\times (Q)$.
\end{theorem}
\begin{proof}
We prove by demonstrating how to transform $Q$ into such a $Q'$.  Given any
non-candidate question $q_i$$=$$\compxyc$ in $Q$, we remove it and, when necessary,
replace several questions.  The decisions and choices are
partitioned into the following three mutually exclusive scenarios, which
correspond to violations of the three conditions in Definition~\ref{def:candidate}.

Case (\ref{item:1}): $q_i$ violates condition (\ref{item:1}), i.e.,
$rlt(q_i)\in R^+(Q_{i-1})$.  We simply remove $q_i$ from $Q$,
which does not change $O_\times$, since the transitive closure already contains
the outcome.

Case (\ref{item:2}): $q_i$ conforms to condition (\ref{item:1}) but violates
condition (\ref{item:2}), i.e., $rlt(q_i)\notin R^+(Q_{i-1})$ and
$\obj{x} \notin O_?(Q_{i-1})$.  The proof for this case is similar to a subset of the
proof for the following case(\ref{item:3}).
We omit the complete proof.

Case (\ref{item:3}): $q_i$ conforms to conditions (\ref{item:1}) and (\ref{item:2})
but violates condition (\ref{item:3}), i.e., $rlt(q_i)$$\notin$$R^+(Q_{i-1})$,
$\obj{x}$$\in$$O_?(Q_{i-1})$, $\exists c' \in C :$ \domC{x}{y}{c'}.
There are three possible subcases, as follows.

(\ref{item:3}-1): $rlt(q_i)$$=$\indiffC{x}{y}{c}.  We simply remove
$q_i$ from $Q$ and thus remove \indiffC{x}{y}{c} from $R^+(Q)$.
(We shall explain one exception, for which we replace $q_i$ instead of
removing it.)
The removal of $q_i$ does not change object partitioning and thus does not change $O_\times$, as explained
below.  The difference between $R^+(Q)$ and
$R(Q)$ is due to transitivity.  Since \indiffC{x}{y}{c} does not participate
in transitivity, we only need to consider the direct impact of removing
\indiffC{x}{y}{c} from $R^+(Q)$.
Therefore, \textbf{(1)} with respect to any object \obj{z} that is not \obj{x} or \obj{y},
\indiffC{x}{y}{c} does not have any impact on \obj{z} since it
does not involve \obj{z}.
\textbf{(2)} With regard to \obj{x}, if \obj{x}$\notin$$O_\times(Q)$, removing
\indiffC{x}{y}{c} from $R^+(Q)$ will not move \obj{x} into $O_\times(Q')$;
if \obj{x}$\in$$O_\times(Q)$, then $Q$ must contain comparisons between
\obj{x} and \obj{z}$\neq$\obj{y} such that \domO{z}{x}.
(\obj{z}$\neq$\obj{y}, because case(\ref{item:3}) violates condition (\ref{item:3}),
i.e., \domO{y}{x} is impossible.)
Removing \indiffC{x}{y}{c} from $R^+(Q)$ does not affect
the comparisons between \obj{z} and \obj{x} and thus does not affect $O_\times$.
\textbf{(3)} For \obj{y}, if \obj{y}$\notin$$O_\times(Q)$,
removing \indiffC{x}{y}{c} from $R^+(Q)$ will not move \obj{y} into $O_\times(Q')$;
if \obj{y}$\in$$O_\times(Q)$, then there are three possible situations---(a) If
\obj{y}$\in$$O_\times(Q_{i-1})$, removing \indiffC{x}{y}{c} will not move \obj{y}
out of $Q_\times$ and thus will not change $Q_\times$. (b) If \obj{y}$\in$$O_?(Q_{i-1})$ and \domO{x}{y} was
ruled out before $q_i$, then $Q$ must contain comparisons between \obj{y} and \obj{z}$\neq$\obj{x}
such that \obj{z} dominates \obj{y}. Removing \indiffC{x}{y}{c} from $R^+(Q)$ does not affect
$Q_\times$ since it does not affect the comparisons between \obj{z} and \obj{y};
(c) If \obj{y}$\in$$O_?(Q_{i-1})$ and \domO{x}{y} was not ruled out before $q_i$,
then we replace $q_i$ (instead of removing $q_i$) by \obj{y}$?_{c}$\obj{x}.
Note that \obj{y}$?_{c}$\obj{x} and \obj{x}$?_{c}$\obj{y} (i.e., $q_i$) are
the same question but different with regard to satisfying the three conditions for candidate questions.
Different from $q_i$, \obj{y}$?_{c}$\obj{x} is a candidate question when it is chosen (i.e., with regard to
$Q_{i-1}$)---$rlt(\obj{y} ?_{c} \obj{x})$$\notin$$R^+(Q_{i-1})$,
$\obj{y}$$\in$$O_?(Q_{i-1})$, \domO{x}{y} is not ruled out.

(\ref{item:3}-2): $rlt(q_i)$$=$\domC{x}{y}{c}. We remove $q_i$ and replace some questions in $Q$.
Consider $S$$=$$O_\times(Q) \setminus O_\times(Q$$\setminus$$\{q_i\})$,
i.e., the objects that would be in $O_\times(Q)$ but instead are in $O_?(Q$$\setminus$$\{q_i\})$ due to the removal of \domC{x}{y}{c} form $R^+(Q)$.
The question replacements are for maintaining $O_\times$ intact.
Fig.\ref{fig:proof} eases our explanation of this case.
Suppose $S_1$$=$$\{\obj{y}\} \cup \{ \obj{v}\;|\;$\domC{y}{v}{c}$\in$$R^+(Q)\}$
and $S_2$$=$$\{\obj{x} \}\cup \{\obj{u}\;|\;$\domC{u}{x}{c}$\in$$R^+(Q)\}$.
We can derive that \textbf{(1)} $S$$\subseteq$$S_1$.  The reason is that the outcome $rlt(q_i)$$=$\domC{x}{y}{c} may
have impact on whether other objects dominate \obj{v} only if $\obj{v}$$=$$\obj{y}$ or \domC{y}{v}{c}, i.e., \obj{v}$\in$$S_1$.
For an object \obj{v} not in $S_1$, the removal of $q_i$ cannot possibly move \obj{v} from $O_\times(Q)$ into $O_?(Q$$\setminus$$\{q_i\})$.
\textbf{(2)} $\forall \obj{v}$$\in$$S, \exists \obj{u}$$\in$$S_2$ such that \domO{u}{v} and \domC{u}{v}{c}.
This is because, if there does not exist such a $\obj{u}$, removing $q_i$ cannot possibly move \obj{v}
from $O_\times(Q)$ into $O_?(Q$$\setminus$$\{q_i\})$, which contradicts with \obj{v}$\in$$S$.

\begin{figure}[t]
    \centering
    \includegraphics[width=50mm]{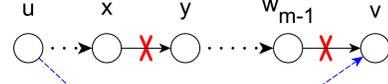}\vspace{-1mm}
    \caption{Question removal and replacement.}
    \label{fig:proof}
\end{figure}

According to the above results, $\exists \obj{w}_1,...,\obj{w}_{m-1},\obj{w}_m$ such that
$\obj{w}_1$$=$$\obj{u}, \obj{w}_m$$=$$\obj{v}, \forall 0$$<$$j$$<$$m$, \domC{w$_j$}{w$_{j+1}$}{c}$\in$$R(Q)$.
In order to make sure \obj{v} stays in $O_\times$, we replace the question
\paircomp{w$_{m-1}$}{v}{c} by \paircomp{v}{u}{c}.  If removing \paircomp{w$_{m-1}$}{v}{c} moves
any object from $O_\times$ into $O_?$, we recursively deal with it as we do for $q_i$.
Note that \paircomp{v}{u}{c} is a candidate question when it succeeds $Q_{i-1}$, since $rlt(\paircomp{v}{u}{c})$
$\notin R^+(Q_{i-1})$ (otherwise \obj{v} does not belong to $S$), \obj{v}$\in$$O_?(Q_{i-1})$ (again, since \obj{v}$\in$$S$),
and \domO{u}{v} cannot be ruled out.

(\ref{item:3}-3): $rlt(q_i)=$ \domC{y}{x}{c}.
This case is symmetric to (\ref{item:3}-2) and so is the proof. We thus omit the details.
\end{proof}

\subsection{Resolving Unusual Contradictions in Question Outcomes}\label{sec:contradict}

A preference relation can be more accurately derived, if more input is
collected from the crowd.  However, under practical
constraints on budget and time, the limited responses from
the crowd ($k$ answers per question) may present two types of
contradicting preferences.

(i) Suppose $\rlt(\compxyc)$$=$\domC{x}{y}{c} and $\rlt(\paircomp{y}{z}{c})$$=$\domC{y}{z}{c}
have been derived, i.e., they belong to $R(Q)$.
They together imply \domC{x}{z}{c}, since a preference relation must be transitive.
Therefore the question \paircomp{x}{z}{c} will not be asked.
If the crowd is nevertheless asked to further compare \obj{x} and \obj{z}, the result
$\rlt(\paircomp{x}{z}{c})$ might be possibly \domC{z}{x}{c}, which presents
a contradiction.

(ii) Suppose $\rlt(\compxyc)$$=$\indiffC{x}{y}{c} and $\rlt(\paircomp{y}{z}{c})$$=$\domC{y}{z}{c}
have been derived from the crowd.  If the crowd is asked to
further compare \obj{x} and \obj{z}, the result $\rlt(\paircomp{x}{z}{c})$ might be
possibly \domC{z}{x}{c}.  The outcomes \domC{y}{z}{c} and \domC{z}{x}{c} together imply
\domC{y}{x}{c}, which contradicts with \indiffC{x}{y}{c}.
(A symmetric case is $\rlt(\compxyc)$$=$\indiffC{x}{y}{c}, $\rlt(\paircomp{y}{z}{c})$$=$\domC{z}{y}{c},
and the crowd might respond with $\rlt(\paircomp{x}{z}{c})$$=$\domC{x}{z}{c},
which also leads to contradiction with \indiffC{x}{y}{c}.
The following discussion applies to this symmetric case, which is thus not
mentioned again.)

In practice, such contradictions are rare, even under just modest number of
answers per question ($k$) and threshold ($\theta$).  This is easy to
understand intuitively---as long as the underlying preference relation
is transitive, the collective wisdom of the crowds will reflect it.
We can find indirect evidence of it in \cite{rorvig-jasis90,carterette-ecir08}, which confirmed that preference judgments of relevance in document retrieval are transitive.
Our empirical results also directly verified it.  We asked Amazon Mechanical Turk
workers to compare $10$ photos by \emph{color}, \emph{sharpness} and \emph{landscape},
and we asked students at our institution to compare $10$ U.S. cities with regard to
\emph{weather}, \emph{living expenses}, and \emph{job opportunities}.
In both experiments, we asked all possible questions---comparing every pair of objects
by every criterion.  For each criterion, we considered the graph representing
the outcomes of questions, where a directed edge represents ``better-than'' and
an undirected edge represents ``indifferent''.  If there is such a ``cycle'' that
it contains zero or one undirected edge and all its directed edges are in the
same direction, the outcomes in the cycle form a contradiction.
We adapted depth-first search to detect all elementary cycles. (A cycle is elementary
if no vertices in the cycle (except the start/end vertex) appear more than once.)
The number of elementary cycles amounts to only 2.9\% and 2.2\% of the number of question
outcomes in the two experiments.
These values would be smaller if we had used larger $k$ and $\theta$.

Nevertheless, contradictions still occur.  Type (i) contradictions can be
prevented by enforcing the following simple Rule~\ref{rule:skip} to assume
transitivity and thus skip certain questions.
They will never get into the derived preference relations.  In fact, in calculating
transitive closure (Definition~\ref{def:closure}) and defining candidate questions
(Sec.\ref{sec:q-select}), we already apply this rule.\vspace{-1mm}

\begin{myrule}[Contradiction Prevention by Skipping Questions] \label{rule:skip}
Given objects \obj{x, y, z} and a criterion $c$, if $\rlt(\compxyc)$$=$\domC{x}{y}{c}
and $\rlt (\paircomp{y}{z}{c})$$=$\domC{y}{z}{c}, we assume $\rlt(\paircomp{x}{z}{c})$$=$\domC{x}{z}{c}
and thus will not ask the crowd to further compare \obj{x} and \obj{z} by criterion $c$. \qed \vspace{-2mm}
\end{myrule}

To resolve type (ii) contradictions, we enforce the following
simple Rule~\ref{rule:resolve}.\vspace{-1mm}

\begin{myrule}[Contradiction Resolution by Choosing Outcomes] \label{rule:resolve}
Consider objects \obj{x, y, z} and a criterion $c$. Suppose $\rlt(\compxyc)$$=$ \indiffC{x}{y}{c} and
$\rlt(\paircomp{y}{z}{c})$$=$\domC{y}{z}{c} are obtained from the crowd.
If $\rlt(\paircomp{x}{z}{c})$$=$\domC{z}{x}{c} is obtained from the crowd afterwards,
we replace the outcome of this question by \indiffC{x}{z}{c}. (Note that we do not
replace it by \domC{x}{z}{c}, since \domC{z}{x}{c} is closer to \indiffC{x}{z}{c}.) \qed
\end{myrule}



\section{Micro-Ordering in Question Selection} \label{sec:alg}
At every iteration of Alg.\ref{alg:general}, we choose a question
\compxyc\ from the set of candidate questions.  By macro-ordering, when
available, we choose a candidate question in which \obj{y} $\notin O_\times$,
i.e., we choose from $Q^1_{can}$.  Otherwise, we choose from $Q^2_{can}$.
The size of $Q^1_{can}$ and $Q^2_{can}$ can be large.  Micro-ordering is for
choosing from the many candidates.  As discussed in
Sec.\ref{sec:framework}, in order to find a short question sequence,
the overriding principle of our question selection strategy is to identify
non-Pareto optimal objects as early as possible.  Guided by this principle,
we discuss several micro-ordering strategies in this section.
Since the strategies are the same for $Q^1_{can}$ and $Q^2_{can}$, we will
simply use the term ``candidate questions'', without distinction between
$Q^1_{can}$ and $Q^2_{can}$.

\subsection{Random Question (\algname{RandomQ})}\vspace{-1mm}
\algname{RandomQ}, as its name suggests, simply selects a random candidate
question.  Table~\ref{table:randomQ} shows an execution
of the general framework under \algname{RandomQ} for Example~\ref{ex:movie}.
For each iteration $i$, the table shows the question outcome $rlt(q_i)$.
Following the question form \compxyc\ in Definition~\ref{def:candidate},
the object ``\obj{x}'' in a question is underlined when we present
the question outcome.
The column ``derived results'' displays derived question outcomes by
transitive closure (e.g., \domC{a}{e}{m} based on $rlt(q_7)$$=$\domC{d}{e}{m}
and $rlt(q_{10})$$=$\domC{a}{d}{m}) and derived object dominance (e.g.,
\domO{b}{d} after $q_{20}$).  The table also shows the object partitions
($O_\surd$, $O_?$ and $O_\times$) when the execution starts and when the
partitions are changed after an iteration.  Multiple iterations may be
presented together if other columns are the same for them.

\begin{table}
\begin{scriptsize}
\begin{center}
\begin{tabular}{|@{}c@{}|@{}c@{}|@{}c@{}|@{}c@{}|@{}c@{}|@{}c@{}|}
 \hline
 $i$ & $rlt(q_i)$ & Derived Results & $O_\surd$ & $O_?$ & $O_\times$ \\ \hline
 $1$-$9$ & \domC{b}{\underline{e}}{m}, \indiffC{\underline{c}}{d}{a}, \indiffC{\underline{a}}{c}{m} & & $\emptyset$ & \{\obj{a,b,c,d,e,f}\} & $\emptyset$ \\
 & \domC{\underline{c}}{e}{s}, \indiffC{\underline{b}}{d}{s}, \domC{b}{\underline{a}}{a} & & & & \\
 & \domC{\underline{d}}{e}{m}, \indiffC{\underline{b}}{d}{m}, \domC{\underline{b}}{f}{s} & & & & \\ \hline
 $10$ & \domC{\underline{a}}{d}{m} & \domC{a}{e}{m} & & & \\ \hline
 $11$ & \domC{c}{\underline{a}}{a} & & & & \\ \hline
 $12$ & \indiffC{\underline{b}}{c}{s} & & & & \\ \hline
 $13$ & \domC{c}{\underline{d}}{m} & \domC{c}{e}{m} & & & \\ \hline
 $14$-$19$ & \domC{d}{\underline{e}}{s}, \indiffC{\underline{e}}{c}{a}, \indiffC{\underline{d}}{f}{a} & & & & \\
 & \indiffC{\underline{a}}{d}{a}, \domC{f}{\underline{a}}{a}, \domC{\underline{b}}{e}{a} & & & & \\ \hline
 $20$ & \domC{b}{\underline{d}}{a} & \domO{b}{d} & $\emptyset$ & \{\obj{a,b,c,e,f}\} & \{\obj{d}\} \\ \hline
 $21$-$23$ & \domC{c}{\underline{f}}{s}, \domC{a}{\underline{e}}{s}, \indiffC{\underline{f}}{b}{m} & & & & \\ \hline
 $24$ & \domC{a}{\underline{f}}{s} & \indiffO{a}{f} & & & \\ \hline
 $25$ & \domC{\underline{e}}{f}{a} & \domC{b}{f}{a}, \domC{b}{a}{a} & $\emptyset$ & \{\obj{a,b,c,e}\} & \{\obj{d,f}\} \\
 & & \domC{e}{a}{a}, \domO{b}{f} & & & \\ \hline
 $26$ & \domC{\underline{b}}{c}{a} & & & & \\ \hline
 $27$ & \domC{a}{\underline{b}}{m} & & & & \\ \hline
 $28$ & \domC{b}{\underline{e}}{s} & \domO{b}{e} & $\emptyset$ & \{\obj{a,b,c}\} & \{\obj{d,e,f}\}\\ \hline
 $29$ & \domC{\underline{c}}{a}{s} & \domC{c}{e}{s}, \domO{c}{a} & $\emptyset$ & \{\obj{b,c}\} & \{\obj{a,d,e,f}\} \\ \hline
 $30$ & \indiffC{\underline{b}}{c}{m} & \domO{b}{c} & \{\obj{b}\} & $\emptyset$ & \{\obj{a,c,d,e,f}\} \\ \hline
\end{tabular}
\end{center}
\end{scriptsize}
\vspace{-3mm}
\caption{\small{\algname{RandomQ} on Example~\ref{ex:movie}.}} 
\label{table:randomQ}
\end{table}

As Table~\ref{table:randomQ} shows, this particular execution under
\algname{RandomQ} requires $30$ questions.  When the execution terminates,
it finds the only Pareto-optimal object \obj{b}.  This simplest
micro-ordering strategy (or rather no strategy at all) already avoids
many questions in the brute-force approach.  The example clearly demonstrates
the benefits of choosing candidate questions only and applying macro-strategy.


\subsection{Random Pair (\algname{RandomP})}\vspace{-1mm}
\algname{RandomP} randomly selects a pair of objects \obj{x} and \obj{y} and
keeps asking questions to compare them (\compxyc\ or \paircomp{y}{x}{c})
until there is no such candidate question, upon which it randomly picks another pair of
objects.  This strategy echoes our principle of eagerly identifying
non-Pareto optimal objects.  In order to declare an object \obj{x} non-Pareto
optimal, we must identify another object \obj{y} such that \obj{y} dominates
\obj{x}.  If we directly compare \obj{x} and \obj{y}, it requires comparing
them by every criterion in $C$ in order to make sure \domO{y}{x}.
By skipping questions according to
transitive closure, we do not need to directly compare them by every
criterion.  However, Property~\ref{property:minquestion} below states
that we still need at least $|C|$ questions involving \obj{x}---some
are direct comparisons with \obj{y}, others are comparisons with other objects
which indirectly lead to outcomes of comparisons with \obj{y}.  When there
is a candidate question \compxyc, it means \obj{y} may dominate \obj{x}.
In such a case, the fewer criteria remain for comparing them, the more likely
\obj{y} will dominate \obj{x}.  Hence, by keeping comparing the same
object pair, \algname{RandomP} aims at finding more
non-Pareto objects by less questions.

\begin{property}\label{property:minquestion}
Given a set of criteria $C$ and an object \obj{x}$\in$$O$, at least $|C|$
pairwise comparison questions involving \obj{x} are required in order to
find another object \obj{y} such that \domO{y}{x}.  
\end{property}
\begin{proof}
By the definition of object dominance, if \domO{y}{x}, then $\forall c$$\in$$C$,
either \domC{y}{x}{c}$\in$$R^+(Q)$ or \indiffC{x}{y}{c}$\in$$R^+(Q)$, and
$\exists c$$\in$$C$ such that \domC{y}{x}{c}$\in$$R^+(Q)$.  Given any particular
$c$, if \indiffC{x}{y}{c}$\in$$R^+(Q)$, then
\indiffC{x}{y}{c}$\in$$R(Q)$, i.e., a question \compxyc\ or
\paircomp{y}{x}{c} belongs to the sequence $Q$, because indifference of
objects on a criterion cannot be derived by transitive closure.
If \domC{y}{x}{c}$\in$$R^+(Q)$, then \domC{y}{x}{c}$\in$$R(Q)$ or
$\exists\:\obj{w}_1, \ldots, \obj{w}_m$$\in$$O$ such that \obj{y}$\succ_c$\obj{w}$_1$$\in$$R(Q),
\ldots$, $\obj{w}_i$$\succ_c$$\obj{w}_{i+1}$$\in$$R(Q), \ldots$, \obj{w}$_m$$\succ_c$\obj{x}$\in$$R(Q)$.
Either way, at least one question involving \obj{x} on each criterion $c$ is required.
Thus, it takes at least $|C|$ questions involving \obj{x} to determine \domO{y}{x}.\vspace{1mm}
\end{proof}

\begin{table}
\begin{scriptsize}
\begin{center}
\begin{tabular}{|@{}c@{}|@{}c@{}|@{}c@{}|@{}c@{}|@{}c@{}|@{}c@{}|}
\hline
 $i$ & $rlt(q_i)$ & Derived Results & $O_\surd$ & $O_?$ & $O_\times$ \\ \hline
 $1$&\domC{c}{\underline{f}}{s}&&$\emptyset$&\{\obj{a,b,c,d,e,f}\} &$\emptyset$\\ \hline
 $2$&\domC{\underline{f}}{c}{m}&\indiffO{f}{c}& & & \\ \hline
 $3-4$ & \domC{\underline{a}}{e}{s}, \domC{\underline{a}}{e}{m} & & & & \\ \hline
 $5$ & \domC{e}{\underline{a}}{a} & \indiffO{a}{e} & & & \\ \hline
 $6-7$ & \domC{c}{\underline{e}}{s}, \domC{c}{\underline{e}}{m} & & & & \\ \hline
 $8$ & \indiffC{\underline{e}}{c}{a} & \domO{c}{e} & $\emptyset$ & \{\obj{a,b,c,d,f}\} & \{\obj{e}\}\\ \hline
 $9$&\domC{\underline{b}}{a}{s}&\domC{b}{e}{s} & & & \\ \hline
 $10$&\domC{a}{\underline{b}}{m}&\indiffO{a}{b} & & & \\ \hline
 $11$&\domC{\underline{d}}{f}{s}& & & & \\ \hline
 $12$&\domC{f}{\underline{d}}{m}&\indiffO{f}{d} & & & \\ \hline
 $13$&\domC{d}{\underline{a}}{s}&\domC{d}{e}{s} & & & \\ \hline
 $14$&\domC{\underline{a}}{d}{m}&\indiffO{a}{d} & & & \\ \hline
 $15-16$ & \indiffC{\underline{b}}{c}{s}, \indiffC{\underline{b}}{c}{m} & & & & \\ \hline
 $17$ & \domC{\underline{b}}{c}{a} & \domO{b}{c} & $\emptyset$ & \{\obj{a,b,d,f}\}&\{\obj{c,e}\}\\ \hline
 $18-19$ & \indiffC{\underline{d}}{b}{s}, \indiffC{\underline{d}}{b}{m} & & & & \\ \hline
 $20$ & \domC{b}{\underline{d}}{a} & \domO{b}{d} & $\emptyset$ & \{\obj{a,b,f}\} & \{\obj{c,d,e}\}\\ \hline
 $21$&\domC{\underline{a}}{f}{s}&\domC{b}{f}{s} & & & \\ \hline
 $22$&\indiffC{\underline{a}}{f}{m} & & & & \\ \hline
 $23$&\domC{f}{\underline{a}}{a}&\indiffO{a}{f} & & & \\ \hline
 $24$&\indiffC{b}{\underline{f}}{m} & & & & \\ \hline
 $25$&\domC{b}{\underline{f}}{a} & \domC{b}{a}{a}, \domO{b}{f} & \{\obj{b}\} & \{\obj{a}\} & \{\obj{c,d,e,f}\}\\ \hline
 $26-27$&\domC{c}{\underline{a}}{s}, \indiffC{\underline{a}}{c}{m} & & & & \\ \hline
 $28$&\domC{c}{\underline{a}}{a}&\domO{c}{a} & \{\obj{b}\} & $\emptyset$ & \{\obj{a,c,d,e,f}\}\\ \hline
\end{tabular}
\end{center}
\end{scriptsize}
\vspace{-3mm}
\caption{\small{\algname{RandomP} on Example~\ref{ex:movie}.}} 
\label{table:randomP}
\end{table}

Table~\ref{table:randomP} illustrates an execution of \algname{RandomP}
Example~\ref{ex:movie}.  The initial two questions are between \obj{c} and \obj{f}.
Afterwards, it is concluded that \indiffO{c}{f} by Property~\ref{property:nondominance}.
Therefore, \algname{RandomP} moves on to ask $3$ questions between \obj{a}
and \obj{e}.  In total, the execution requires $28$ questions.  Although it is
shorter than Table~\ref{table:randomQ} by only $2$ questions due to the small
size of the example, it clearly moves objects into $O_\times$ more quickly.
(In Table~\ref{table:randomQ}, $O_\times$ is empty until the $20$th question.
In Table~\ref{table:randomP}, $O_\times$ already has $3$ objects after $20$ questions.)
The experiment results in Sec.\ref{sec:exp} exhibit significant performance
gain of \algname{RandomP} over \algname{RandomQ} on larger data.

\subsection{Pair Having the Fewest Remaining Questions (\algname{FRQ})}\vspace{-1mm}
Similar to \algname{RandomP}, once a pair of objects \obj{x} and \obj{y}
are chosen, \algname{FRQ} keeps asking questions between \obj{x} and
\obj{y} until there is no such candidate questions.  Different from
\algname{RandomP}, instead of randomly picking a pair of objects,
\algname{FRQ} always chooses a pair with the fewest remaining questions.
There may be multiple such pairs.  To break ties, \algname{FRQ} chooses
such a pair that \obj{x} has dominated the fewest other objects and
\obj{y} has dominated the most other objects.  Furthermore, in comparing
\obj{x} and \obj{y}, \algname{FRQ} orders their remaining questions
(and thus criteria) by how likely \obj{x} is worse than \obj{y} on the
criteria.  Below we explain this strategy in more detail.

\vspace{-1mm}
{\flushleft \textbf{Selecting Object Pair}}\hspace{2mm}
Consider a question sequence $Q_i$ so far and \algname{FRQ} is to select
the next question $Q_{i+1}$.  We use $C_{\obj{x,y}}$ to denote the set of criteria $c$ such that
\compxyc\ is a candidate question, i.e., $C_{\obj{x,y}}$$=$$\{c$$\in$$C\; |\;
\compxyc$$\in$$Q^1_{can}\}$.  (We assume $Q^1_{can}$ is not empty.
Otherwise, \algname{FRQ} chooses from $Q^2_{can}$ in the same way;
cf. Alg.\ref{alg:qselect}.)  By Definition~\ref{def:candidate}, the
outcomes of these questions are unknown, i.e., $\forall$$c$$\in$$C_{\obj{x,y}}:rlt(\compxyc)$$\notin$$ R^+(Q_i)$.  Furthermore, if any
remaining question (whose outcome is unknown) between \obj{x} and \obj{y}
is a candidate question, then all remaining questions between them are
candidate questions.  \algname{FRQ} chooses a pair with the fewest
remaining candidate questions, i.e., a pair belonging to
$S_1$$=$$\argmin_{(\obj{x,y})} |C_{\obj{x,y}}|$.

The reason to choose such a pair is intuitive.  It requires at least
$|C_{\obj{x,y}}|$ candidate questions to determine \domO{y}{x}.  (The
proof would be similar to that of Property~\ref{property:minquestion}.)
Therefore, $\min_{(\obj{x,y})} |C_{\obj{x,y}}|$ is the minimum number of
candidate questions to further ask, in order to determine that an object
is dominated, i.e., non-Pareto optimal.  Thus, a pair in
$S_1$ may lead to a dominated object by the fewest questions, matching our
goal of identifying non-Pareto optimal objects as soon as possible.

We can further justify this strategy in a probabilistic sense.  For
\domO{y}{x} to be realized, it is necessary that none of the remaining
questions has an outcome \domC{x}{y}{c}, i.e., $\forall c$$\in$$C_{\obj{x,y}} :
rlt(\compxyc)$ $\neq$ \domC{x}{y}{c}.  Make the simplistic
assumption that every question \compxyc\ has an equal probability $p$ of
not having outcome \domC{x}{y}{c}, i.e., $\forall \compxyc$$\in$$Q^1_{can}$,
$P(rlt(\compxyc)$$\neq$\domC{x}{y}{c})$=$$p$.  Further assuming independence of
question outcomes, the probability of satisfying the aforementioned necessary condition is
$p^{|C_{\obj{x,y}}|}$.  By taking a pair belonging to
$S_1$, we have the largest probability
of finding a dominated object.  We note that, for \domO{y}{x} to be
realized, in addition to the above necessary condition, another condition
must be satisfied---if $\nexists c$ such that \domC{y}{x}{c}
$\in R^+(Q_i)$, the outcome of at least one remaining question should be
\domC{y}{x}{c}, i.e., $\exists c$$\in$$C_{\obj{x,y}} : rlt(\compxyc)$$=$\domC{y}{x}{c}.
Our informal probability-based analysis does not consider this extra
requirement.

\vspace{-1mm}
{\flushleft \textbf{Breaking Ties}}\hspace{2mm}
There can be multiple object pairs with the fewest remaining questions,
i.e., $|S_1|$$>$$1$.
To break ties, \algname{FRQ} chooses such an \obj{x} that has dominated the
fewest other objects, since it is more likely to be dominated.  If there
are still ties, \algname{FRQ} further chooses such a \obj{y} that has
dominated the most other objects, since it is more
likely to dominate \obj{x}.  More formally, \algname{FRQ} chooses a pair
belonging to $S_2$$=$$\{(\obj{x,y})$$\in$$S_1 \; |\; \nexists (\obj{x',y'})$$\in$$S_1
\text{ such that } d(\obj{x'})$$>$$d(\obj{x}) \vee (d(\obj{x'})$$=$ $d(\obj{x})
\wedge d(\obj{y'})$$>$$d(\obj{y}))\}$, where the function $d(\cdot)$ returns
the number of objects so far dominated by an object, i.e., $\forall \obj{x},
d(\obj{x})=|\{\obj{y} | $\domO{x}{y}$\text { based on } R^+(Q_i)\}|$.
This heuristic follows the principle of detecting non-Pareto optimal objects
as early as possible.  Note that $S_2$ may still contain multiple object
pairs.  In such a case, \algname{FRQ} chooses an arbitrary pair.

\vspace{-1mm}
{\flushleft \textbf{Selecting Comparison Criterion}}\hspace{2mm}
Once a pair (\obj{x,y}) is chosen, \algname{FRQ} has to select a criterion for the next question.
\algname{FRQ} orders the remaining criteria $C_{\obj{x,y}}$ based on the heuristic that the sooner it
understands \domO{y}{x} will not happen, the lower cost it pays.
As discussed before, $|C_{\obj{x,y}}|$ questions are required in order to
conclude that \domO{y}{x}; on the other hand, only one question (if asked
first) can be enough for ruling it out.  Consider the case that \obj{x} is
better than \obj{y} by only one remaining criterion, i.e.,
$\exists c$$\in$$C_{\obj{x,y}} : rlt(\compxyc)$$=$\domC{x}{y}{c} and
$\forall c'$$\in$$C_{\obj{x,y}}$, $c'$$\neq$$c : rlt(\obj{x} ?_{c'} \obj{y})$$=$$\obj{x}$$\nsucc_{c'}$$ \obj{y}$.
If \algname{FRQ} asks \compxyc\ after all other
remaining questions, it takes $|C_{\obj{x,y}}|$ questions to understand
\obj{y} does not dominate \obj{x}; but if \compxyc\ is asked first, no more
questions are necessary, because there will be no more candidate questions
in the form of \compxyc.

Therefore, \algname{FRQ} orders the criteria $C_{\obj{x,y}}$ by a scoring
function that reflects the likelihood of \obj{x}'s superiority than \obj{y}
by the corresponding criteria.  More specifically, for each $c$$\in$$C_{\obj{x,y}}$, its score is
$r_c(\obj{x,y})$$=$$r_c(\obj{y})$$+$${r'}_c(\obj{y})$$-$${r''}_c(\obj{y})$$-$$(r_c(\obj{x})$$+$${r'}_c(\obj{x})$$-$${r''}_c(\obj{x}))$
where $r_c(\obj{y})$$=$$|\{$\obj{z} $|$ \domC{z}{y}{c}$\in$$R^+(Q_i)\}|$,
${r'}_c(\obj{y})$$=$$|\{$\obj{z} $|$ \indiffC{y}{z}{c}$\in$$R^+(Q_i)\}|$,
and ${r''}_c(\obj{y})$$=$$|\{$\obj{z} $|$ \domC{y}{z}{c}$\in$$R^+(Q_i)\}|$.
In this scoring function, $r_c(\obj{y})$ is the number of objects preferred over
\obj{y} by criterion $c$, ${r'}_c(\obj{y})$ is the number of objects equally good
(or bad) as \obj{y} by $c$, and ${r''}_c(\obj{y})$ is the number of objects to
which \obj{y} is preferred with regard to $c$.
\algname{FRQ} asks the remaining questions in decreasing order of the
corresponding criteria's scores.  This way, it may find such a question
that $rlt(\compxyc)$$=$\domC{x}{y}{c} earlier than later.

\begin{table}
\begin{scriptsize}
\begin{center}
\begin{tabular}{|@{}c@{}|@{}c@{}|@{}c@{}|@{}c@{}|@{}c@{}|@{}c@{}|@{}c@{}|}
 \hline
 $i$ & $rlt(q_i)$ & Derived Results &(\obj{x,y}), $C_{\obj{x,y}}$ & $O_\surd$ & $O_?$ & $O_\times$ \\ \hline
 & & & (\obj{a,b}), $\{s,m,a\}$ & $\emptyset$ & \{\obj{a,b,c,d,e,f}\}&$\emptyset$\\ \hline
 $1$&\domC{b}{\underline{a}}{s}& & (\obj{a,b}), $\{m,a\}$ & & & \\ \hline
 $2$&\domC{\underline{a}}{b}{m} & \indiffO{a}{b} & (\obj{a,c}), $\{s,a,m\}$ & & & \\ \hline
 $3$&\domC{c}{\underline{a}}{s} & & (\obj{a,c}), $\{a,m\}$ & & & \\ \hline
 $4$&\indiffC{c}{\underline{a}}{a} & & (\obj{a,c}), $\{m\}$ & & & \\ \hline
 $5$&\domC{c}{\underline{a}}{m}&\domO{c}{a} & (\obj{b,c}), $\{a,s,m\}$ & $\emptyset$&\{\obj{b,c,d,e,f}\}&\{\obj{a}\} \\ \hline
 $6$&\indiffC{\underline{b}}{c}{a} & & (\obj{b,c}), $\{s,m\}$ & & & \\ \hline
 $7$&\indiffC{\underline{b}}{c}{s} & & (\obj{b,c}), $\{m\}$ & & & \\ \hline
 $8$&\domC{\underline{b}}{c}{m} & \domC{b}{a}{m}, \domO{b}{c} & (\obj{d,b}), $\{a,s,m\}$ & $\emptyset$ & \{\obj{b,d,e,f}\} & \{\obj{a,c}\} \\ \hline
 $9$&\indiffC{b}{\underline{d}}{a} & & (\obj{d,b}), $\{s,m\}$ & & & \\ \hline
 $10$&\indiffC{b}{\underline{d}}{s} & & (\obj{d,b}), $\{m\}$ & & & \\ \hline
 $11$&\domC{b}{\underline{d}}{m} & \domO{b}{d} & (\obj{e,b}), $\{a,s,m\}$ & $\emptyset$&\{\obj{b,e,f}\}&\{\obj{a,c,d}\} \\ \hline
 $12$&\domC{b}{\underline{e}}{a} & & (\obj{e,b}), $\{s,m\}$ & & & \\ \hline
 $13$&\domC{b}{\underline{e}}{s} & & (\obj{e,b}), $\{m\}$ & & & \\ \hline
 $14$&\domC{b}{\underline{e}}{m} & \domC{a}{e}{m}, \domO{b}{e} & (\obj{f,b}), $\{a,s,m\}$ & $\emptyset$ & \{\obj{b,f}\} & \{\obj{a,c,d,e}\} \\ \hline
 $15$&\domC{b}{\underline{f}}{a} & & (\obj{f,b}), $\{s,m\}$ & & & \\ \hline
 $16$&\domC{b}{\underline{f}}{s} & & (\obj{f,b}), $\{m\}$ & & & \\ \hline
 $17$&\domC{b}{\underline{f}}{m}&\domO{b}{f} & & \{\obj{b}\} & $\emptyset$ & \{\obj{a,c,d,e,f}\} \\ \hline
\end{tabular}
\end{center}
\end{scriptsize}
\vspace{-3mm}
\caption{\small{\algname{FRQ} on Example~\ref{ex:movie}.}} 
\label{table:frq}
\end{table}

Table~\ref{table:frq} presents the framework's execution for
Example~\ref{ex:movie}, by applying the \algname{FRQ} policy.  In addition
to the same columns in Tables~\ref{table:randomQ} and~\ref{table:randomP},
Table~\ref{table:frq} also includes an extra column to show, at each
iteration, the chosen object pair for the next question
(\obj{x,y}) and the set of remaining comparison criteria between them
($C_{\obj{x,y}}$).  The criteria in $C_{\obj{x,y}}$ are ordered by the
aforementioned ranking function $r(\cdot)$.  At the beginning of the
execution, the object pair is arbitrarily chosen and the criteria are
arbitrarily ordered.  In the example, we assume \paircomp{a}{b}{s} is chosen as 
the first question.  After $q_2$, \algname{FRQ} can derive that \indiffO{a}{b}.
Hence, there is no more candidate question between them and \algname{FRQ}
chooses the next pair (\obj{a,c}).  Three questions are asked for comparing
them.  At the end of $q_5$, multiple object pairs have the fewest remaining
questions.  By breaking ties, (\obj{b,c}) is chosen as the next pair, since
only \obj{c} has dominated any object so far.  The remaining criteria
$C_{\obj{b,c}}$ are ordered as $\{a,s,m\}$, because 
$r_a(\obj{b,c})$$>$$r_s(\obj{b,c})$ and $r_a(\obj{b,c})$$>$$r_m(\obj{b,c})$. 
The execution sequence terminates after $17$ questions, much shorter than 
the $30$ and $28$ questions by \algname{RandomQ} and \algname{RandomP}, respectively. 

To conclude the discussion on micro-ordering, we derive
a lower bound on the number of questions required for finding all Pareto-optimal objects
(Theorem~\ref{th:lowerbound}).
The experiment results in Sec.\ref{sec:exp} reveal that \algname{FRQ} is nearly optimal
and the lower bound is practically tight, since the number of questions used by \algname{FRQ}
is very close to the lower bound.

\begin{theorem}\label{th:lowerbound}
Given objects $O$ and criteria $C$, to find all Pareto-optimal objects in $O$, at least $(|O|$$-$$k)$$\times$$|C|$$+$$(k$$-$$1)$$\times$$2$ pairwise comparison questions are necessary, where $k$ is the number of Pareto-optimal objects in $O$. \vspace{-1mm}
\end{theorem}
\begin{proof}
Suppose the non-Pareto optimal objects are $O_1$ and the Pareto-optimal objects are $O_2$ ($O_1$$\cup$$O_2$$=$$O$ and $O_1$$\cap$$O_2$ $=$$\emptyset$).  
We first separately consider $n_1$ (the minimum number of questions involving objects in $O_1$) and $n_2$ (the minimum number of questions comparing objects within $O_2$ only).

(1) By Property~\ref{property:minquestion} (and its proof), for every non-Pareto optimal object \obj{x}$\in$$O_1$, at least $|C|$ questions involving \obj{x} are required.  There exists at least an object \obj{y} such that \domO{y}{x}. The required $|C|$ questions lead to outcome either \indiffC{y}{x}{c} or \domC{z}{x}{c} such that \obj{y}$\succ_c \ldots \succ_c$\domC{z}{x}{c} (\obj{z} can be \obj{y}) for each $c$$\in$$C$.  For different \obj{x}, the $|C|$ questions cannot overlap---for a question with outcome \domC{z}{x}{c}, the \obj{x} is different; for a question with outcome \indiffC{y}{x}{c}, the same question cannot be part of both the $|C|$ questions for \obj{x} and the $|C|$ questions for \obj{y} to detect both as non-Pareto optimal, because it is impossible that \domO{x}{y} and \domO{y}{x}.  Hence, $n_1$$=$$(|O|$$-$$k)$$\times$$|C|$.

(2) Given any Pareto-optimal object \obj{x}$\in$$O_2$, for any other \obj{y}$\in$$O_2$, either (a) \indiffC{x}{y}{c} for all criteria $c$$\in$$C$ or (b) there exist at least two criteria $c_1$ and $c_2$ such that \domC{x}{y}{c_1} and \domC{y}{x}{c_2}.  Among the $k$$-$$1$ other objects in $O_2$, suppose $k_a$ and $k_b$ of them belong to cases (a) and (b), respectively ($k_a$$+$$k_b$$=$$k$$-$$1$).
Under case (a), each of the $k_a$ objects requires $|C|$ questions.
Under case (b), there must be a question leading to outcome \domC{z}{y}{c_1}, where \obj{z}$=$\obj{x} or \obj{x}$\succ_{c_1} \ldots \succ_{c_1}$\domC{z}{y}{c_1}.  Similarly, there must be a question with outcome \domC{y}{z}{c_2} such that \obj{z}$=$\obj{x} or \domC{y}{z}{c_2}$\succ_{c_2} \ldots \succ_{c_2}$\obj{x}.  Therefore, each of the $k_b$ objects requires at least $2$ questions.
Clearly, such required questions for comparing \obj{x} with the $k$$-$$1$ other objects in $O_2$ are all distinct.  They are also all different from the questions involving non-Pareto optimal objects (case (1)). Hence, $n_2$$=$$k_a$$\times$$|C|$$+$$k_b$$\times$$2$ $\geq$ $(k$$-$$1)$$\times$$2$.

Summing up $n_1$ and $n_2$, a lower bound on the number of required questions is thus $(|O|$$-$$k)$$\times$$|C|$$+$$(k-1)$$\times$$2$.
Note that, when $k$$=$$0$, a trivial, tighter lower bound is $|O|$$\times$$|C|$. (One example in which $k$$=$$0$ is Fig.\ref{fig:intransitivity}.)
\end{proof}

\vspace{2mm}
\section{Experiments and Case Studies}\label{sec:exp}

We designed and conducted experiments to compare the efficiency of different
instantiations of the general framework under varying problem sizes, by
simulations on a large dataset.  We also investigated two case studies by using
a set of human judges and a real crowdsourcing marketplace.

\subsection{Efficiency and Scalability Test by Simulation}\label{sec:exp-simulate}
In this experiment we studied the efficiency and scalability of various
instantiations of the general framework.  Given the large number of questions
required for such a study, we cannot afford using a real crowdsourcing
marketplace.  Hence, we performed the following simulation.
Each object is an NBA player in a particular year.
The objects are compared by $10$ criteria, i.e., performance categories such as
\emph{points}, \emph{rebounds}, \emph{assists}, etc.
We simulated the corresponding $10$ preference relations based
on the players' real performance in individual years, as follows.
Consider a performance category $c$ and two objects \obj{x}=(player1, year1)
and \obj{y}=(player2, year2).  Suppose \obj{x}.$c$$>$\obj{y}.$c$,  where
\obj{x}.$c$ is player1's per-game performance on category $c$ in year1 (similarly
for \obj{y}.$c$). We generated a uniform random number $v$ in [$0,1$].
If $v$$<$$1-e^{-(x.c-y.c)}$, we set \domC{x}{y}{c}, otherwise we set \indiffC{x}{y}{c}.
This way, we introduced a perturbation into the preference relations in order
to make sure they are partial orders, as opposed to directly using real
performance statistics (which would imply bucket orders).
Fig.\ref{fig:heatmap} shows that the number of Pareto-optimal objects
increases by the sizes of both object set $O$ (objects are randomly selected) and criteria set $C$
(the first $|C|$ criteria of the aforementioned $10$ criteria).

\begin{figure}[t]
    \centering
    \includegraphics[width=50mm]{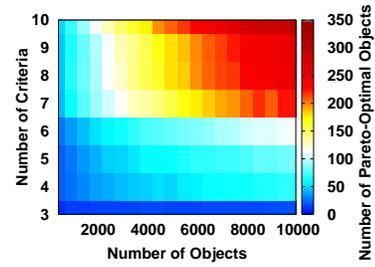}\vspace{-2mm}
    \caption{\small{Number of Pareto-optimal objects by $|O|$ and $|C|$.}}
    \label{fig:heatmap}
    \vspace{-1mm}
\end{figure}

\begin{figure}[t]
            \centering
        \begin{subfigure}[b]{0.5\textwidth}
                \includegraphics[width=\textwidth]{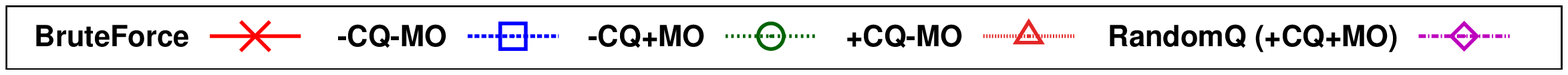}
        \end{subfigure}\vspace{-2mm}

        ~ 
        \hspace{-2mm}\begin{subfigure}[b]{0.25\textwidth}
                \includegraphics[width=\textwidth]{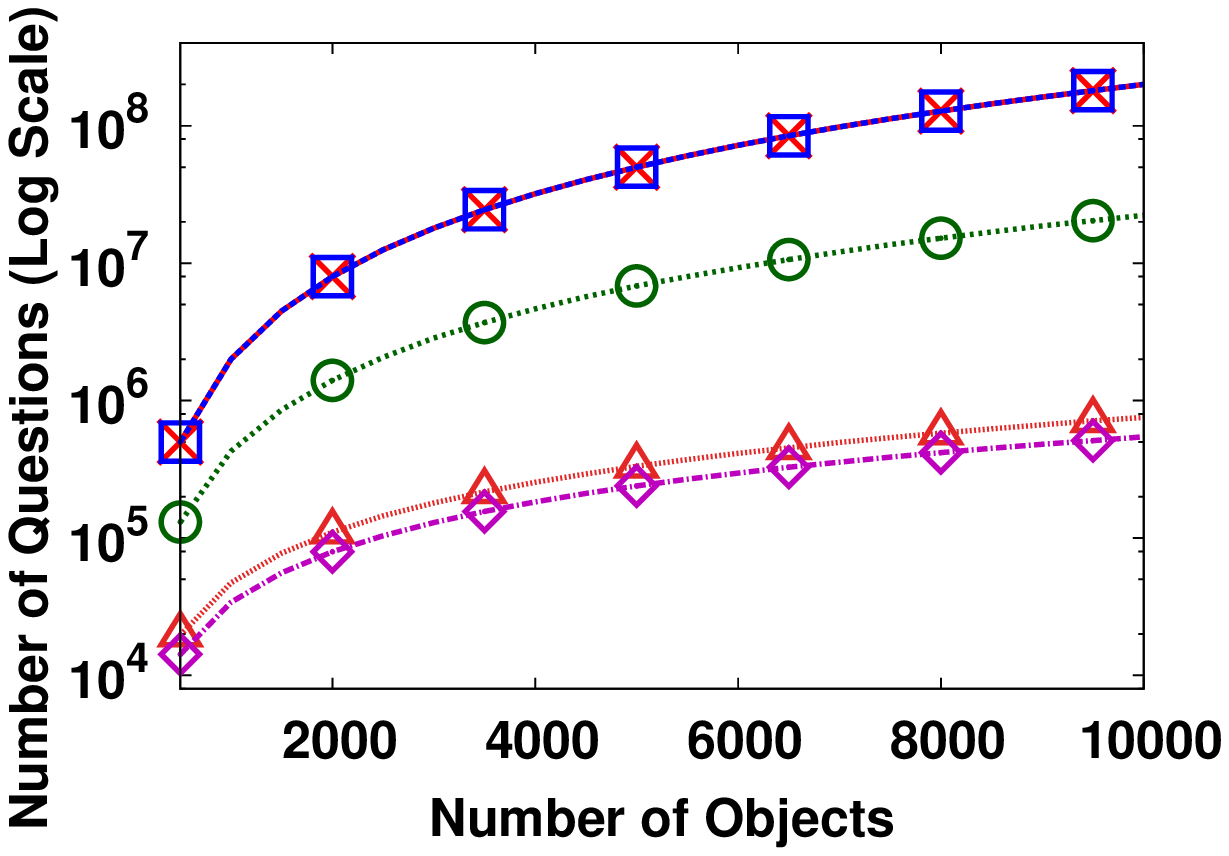}
                \caption{\small{$|C|=4$, varying $|O|$}}
                \label{fig:c4prime}
        \end{subfigure}%
        ~
        \hspace{-2mm}\begin{subfigure}[b]{0.25\textwidth}
                \includegraphics[width=\textwidth]{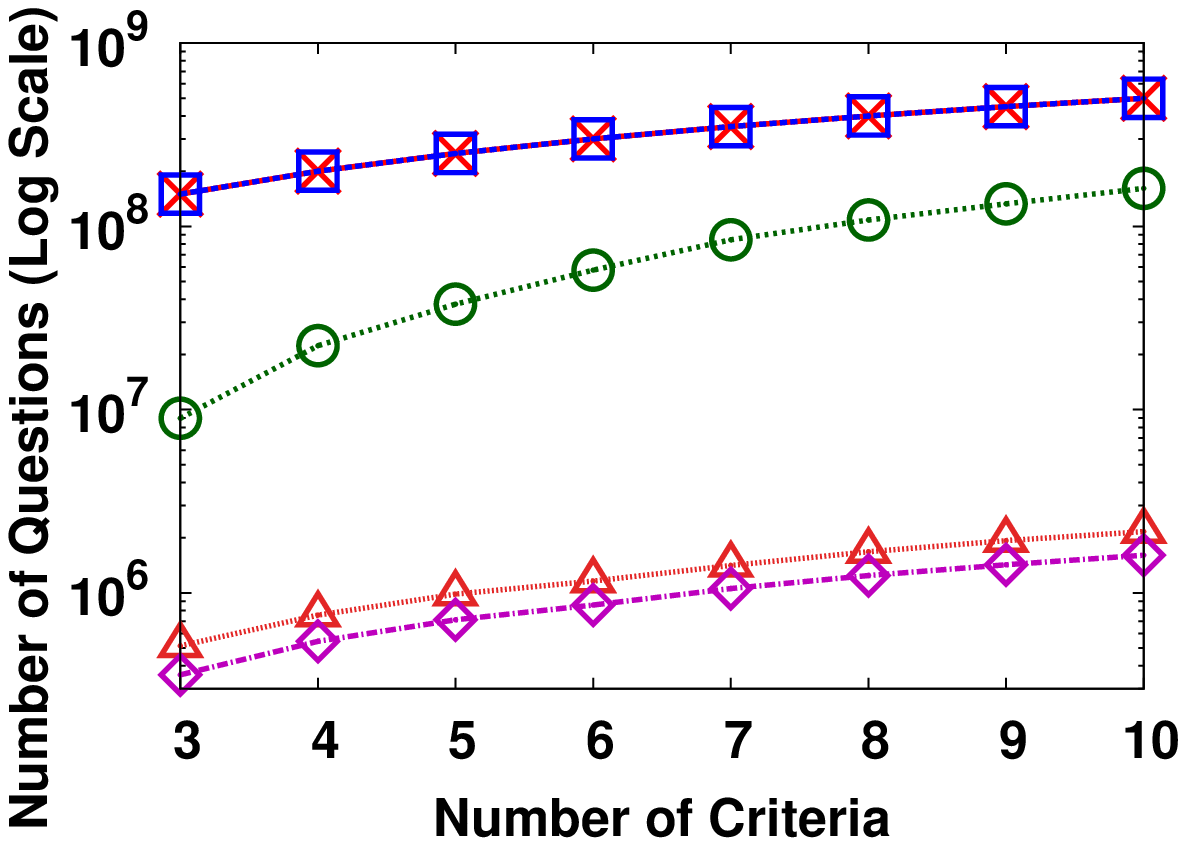}
                \caption{\small{$|O|=10,000$, varying $|C|$}}
                \label{fig:o10000prime}
        \end{subfigure}
    \vspace{-5mm}
    \caption{\small{Numbers of questions by \salgname{BruteForce} and four basic methods.}}
    \label{fig:NBA2}
\end{figure}

\begin{figure}[t]
            \centering
        \begin{subfigure}[b]{0.5\textwidth}
                \includegraphics[width=\textwidth]{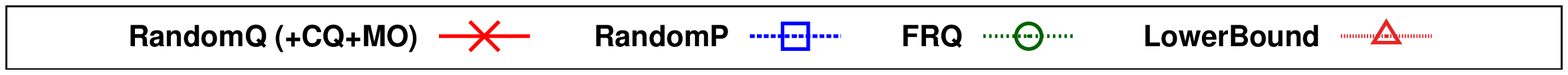}
        \end{subfigure}\vspace{-2mm}

        ~ 
        \begin{subfigure}[b]{0.25\textwidth}
                \includegraphics[width=\textwidth]{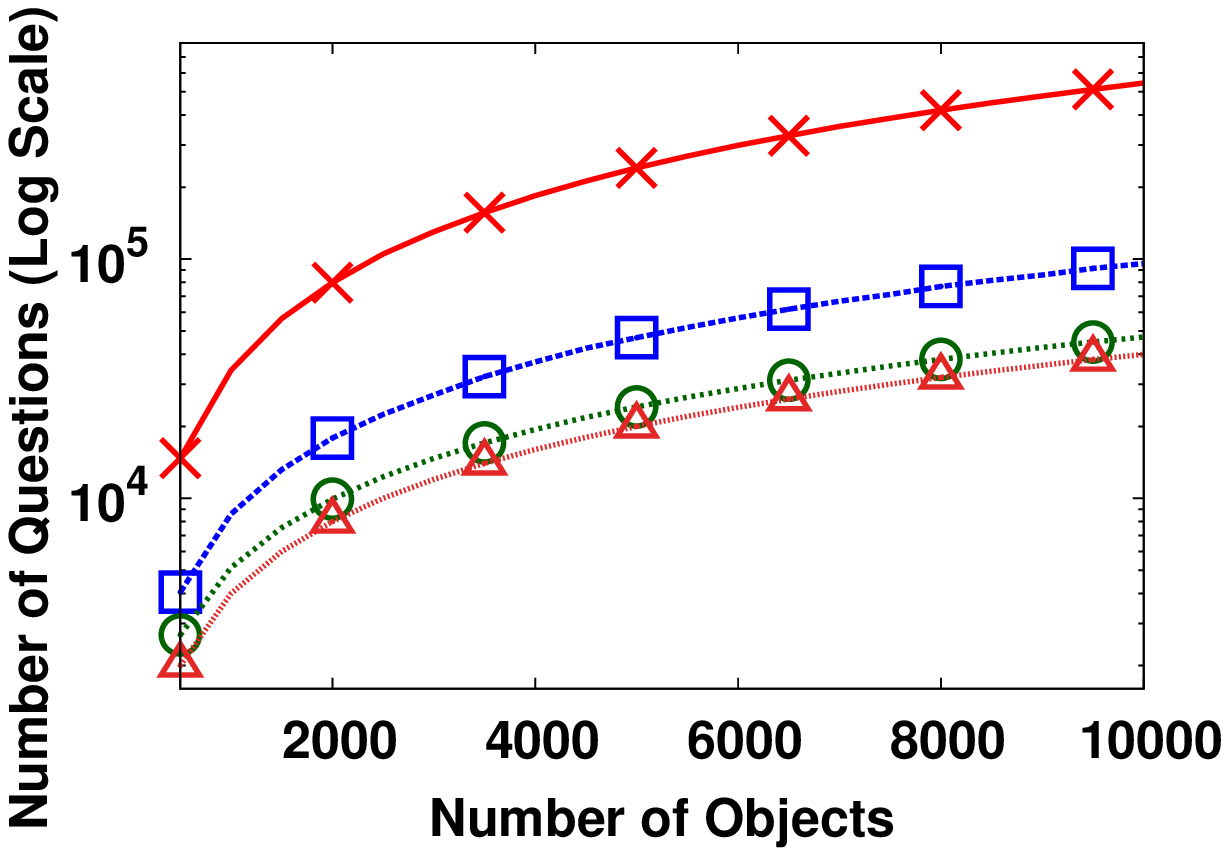}
                \caption{\small{$|C|=4$, varying $|O|$}}
                \label{fig:c4}
        \end{subfigure}%
        ~
        \hspace{-2mm}\begin{subfigure}[b]{0.25\textwidth}
                \includegraphics[width=\textwidth]{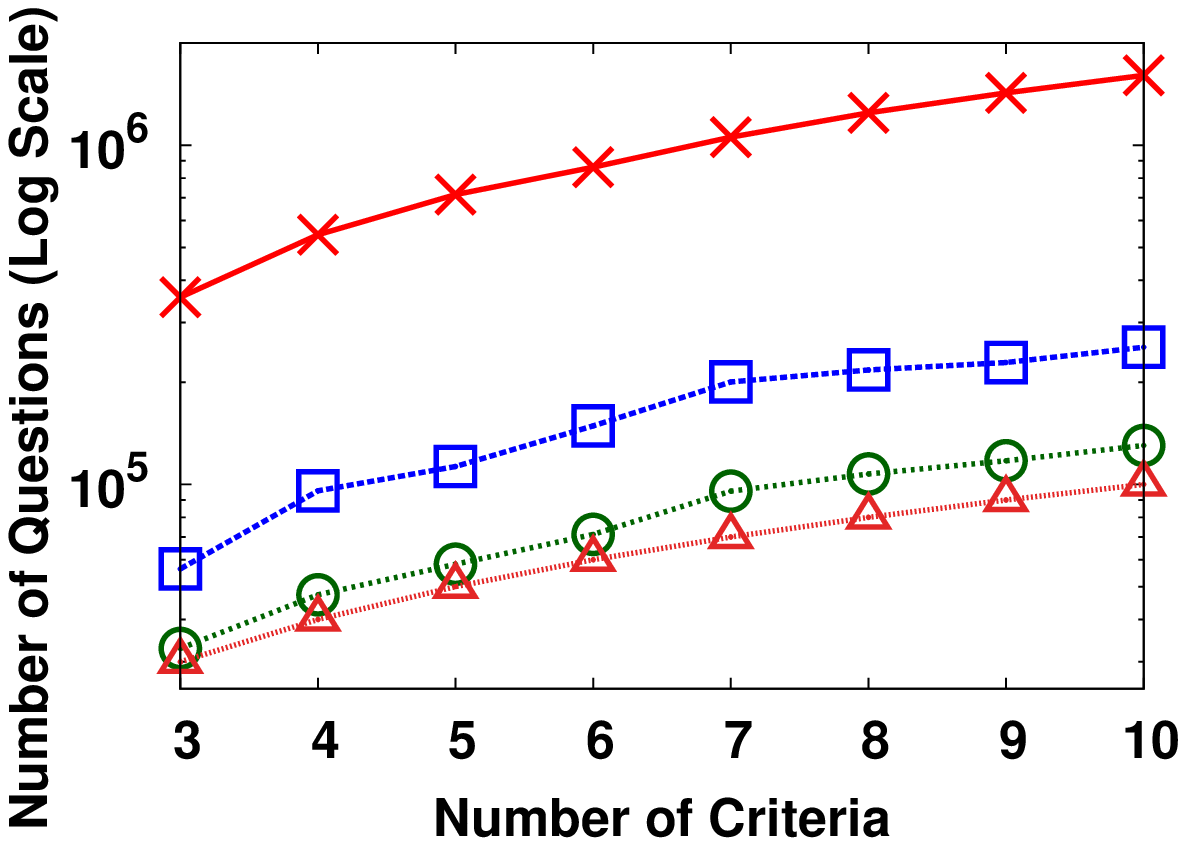}
                \caption{\small{$|O|=10,000$, varying $|C|$}}
                \label{fig:o10000}
        \end{subfigure}
    \vspace{-5mm}
    \caption{\small{Numbers of questions by different micro-ordering heuristics.}}
    \label{fig:NBA1}
\end{figure}

\subsubsection{Effectiveness of candidate questions and macro-ordering}\label{sec:CQMO}
To verify the effectiveness of candidate questions and macro-ordering,
we compared five methods---\algname{BruteForce}, \algname{--CQ--MO}, \algname{--CQ+MO}, \algname{+CQ--MO}, and \algname{+CQ+MO}.
The notation +/-- before \algname{CQ} and \algname{MO} indicates whether a method only selects candidate questions (\algname{CQ})
and whether it applies the macro-ordering strategy (\algname{MO}), respectively.  In all these five methods,
qualifying questions are randomly selected, i.e., no particular micro-ordering heuristics are applied.
For instance, \algname{+CQ+MO} selects only candidate questions and applies macro-ordering.  Hence, it is
equivalent to \algname{RandomQ}.
Fig.\ref{fig:NBA2} shows the numbers of required pairwise comparisons (in logarithmic scale) for each method,
varying by object set size ($|O|$ from $500$ to $10,000$ for $|C|$$=$$4$) and criterion set size ($|C|$ from $3$ to $10$ for $|O|$$=$$10,000$).
The figure clearly demonstrates the effectiveness of both \algname{CQ} and \algname{MO}, as taking out
either feature leads to significantly worse performance than \algname{RandomQ}.
Particularly, the gap between \algname{+CQ--MO} and \algname{--CQ+MO} suggests that choosing only candidate questions
has more fundamental impact than macro-ordering.
If neither is applied (i.e., \algname{--CQ--MO}), the performance is equally poor as that of \algname{BruteForce}.
(\algname{--CQ--MO} uses slightly less questions than \algname{BruteForce}, since it can terminate before exhausting all questions.
However, the difference is negligible for practical purpose, as their curves overlap under logarithmic scale.)

\subsubsection{Effectiveness of micro-ordering}
Fig.\ref{fig:NBA1} presents the numbers of pairwise comparisons required by different micro-ordering heuristics
(\algname{RandomQ}, i.e., \algname{+CQ+MO}, \algname{RandomP}, \algname{FRQ}) and \algname{LowerBound} (cf. Theorem~\ref{th:lowerbound})
under varying sizes of the object set ($|O|$ from $500$ to $10,000$ for $|C|$$=$$4$)
and the criteria set ($|C|$ from $3$ to $10$ for $|O|$$=$$10,000$).  In all these instantiations of the general framework, \algname{CQ} and \algname{MO} are applied.
The results are averaged across $30$ executions.  All these methods outperformed \algname{BruteForce} by orders of magnitude.
(\algname{BruteForce} is not shown in Fig.\ref{fig:NBA1} since it is off scale, but its number can be calculated by
equation $|C|$$\times$$|O|$$\times$$(|O|-1)/2$.)
For instance, for $5,000$ objects and $4$ criteria, the ratio of pairwise comparisons required by even the
naive \algname{RandomQ} to that used by \algname{BruteForce} is already as low as $0.0048$.
This clearly shows the effectiveness of \algname{CQ} and \algname{MO}, as discussed for Fig.\ref{fig:NBA2}.
The ratios for \algname{RandomP} and \algname{FRQ} are further several times smaller ($0.00094$ and $0.00048$, respectively).
The big gain by \algname{FRQ} justifies the strategy of choosing object pairs with the fewest
remaining questions.  Especially, \algname{FRQ} has nearly optimal performance, because it gets very close to \algname{LowerBound} in Fig.\ref{fig:NBA1}. 
The small gap between \algname{FRQ} and \algname{LowerBound} also indicates that the lower bound is practically tight. 
The figure further suggests excellent scalability of \algname{FRQ} as its number of questions grows linearly by both $|C|$ and $|O|$.


\subsection{Case Studies}

\begin{figure}[t]
    \centering
    \includegraphics[width=85mm]{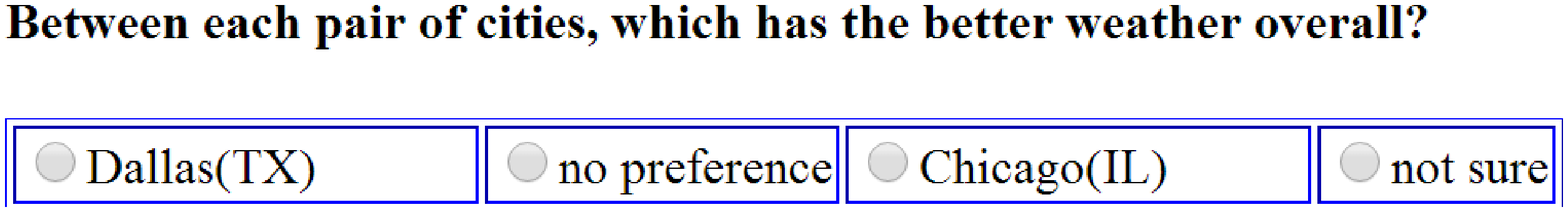}\vspace{-1mm}
    \caption{\small{A question in Case Study 1.}}
    \label{fig:city-form}
\end{figure}

\subsubsection{Case Study 1: collecting opinions from a small group of people (cf. Example~\ref{ex:movie})}
We asked a small number of students at our institution to compare $10$ U.S. cities by \emph{overall weather}, \emph{job opportunities}, and \emph{living expenses}.
The $135$ possible pairwise comparison questions were partitioned into $15$ question forms, each of which contains $9$ questions on a criterion.
Fig.\ref{fig:city-form} shows one such question.  We requested responses from $10$ people for each form.  Some people skipped various questions, allowed by
the ``not sure'' option in Fig.\ref{fig:city-form}.  Eventually we collected at least $6$ responses to each question.  The question outcomes were
derived by setting $\theta$$=$$60\%$ (cf. Equation~(\ref{eq:outcome})).

\begin{table}[t]
\begin{footnotesize}
\begin{center}
\begin{tabular}{|@{\hspace{1mm}}c@{\hspace{1mm}}||@{\hspace{1mm}}c@{\hspace{1mm}}|@{\hspace{1mm}}c@{\hspace{1mm}}|@{\hspace{1mm}}c@{\hspace{1mm}}|@{\hspace{1mm}}c@{\hspace{1mm}}|@{\hspace{1mm}}c@{\hspace{1mm}}|}
\hline
  & \salgname{LowerBound} & \salgname{FRQ} &\salgname{RandomP} & \salgname{RandomQ} & \salgname{BruteForce} \\ \hline
 Case Study 1 & 24 & 48 & 76 & 90 & 135\\ \hline
 Case Study 2 & 25 & 40 & 52 & 77 & 135\\ \hline
\end{tabular}
\end{center}
\end{footnotesize}
\vspace{-3mm}
\caption{\small{Numbers of questions by various methods in case studies.}}\vspace{-2mm}
\label{table:caseStudyResults}
\end{table}

Table~\ref{table:caseStudyResults} shows the numbers of questions taken by various methods in this case study, averaged over 30 executions.
While the performance gap between different methods is far less than in Sec.\ref{sec:exp-simulate},
we note that it is due to the very small object set.
In fact, based on the question outcomes, there are $4$ Pareto-optimal cities out of the $10$ cities, which makes it
less likely for cities to dominate each other and thus requires more questions.
Fig.\ref{fig:cityDom} shows the Hasse diagram for representing the dominance relation between the cities.

\begin{figure}[t]
    \centering
    \includegraphics[width=65mm]{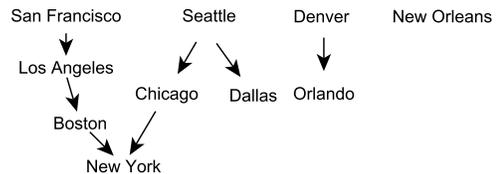}\vspace{-2mm}
    \caption{\small{Hasse diagram for the $10$ U.S. cities in Case Study 1.}}
    \label{fig:cityDom}
\end{figure}

\subsubsection{Case Study 2: information exploration using a crowdsourcing marketplace (cf. Example~\ref{ex:pic})}
This study was conducted in similar manner as Case Study 1.
We used a real crowdsourcing marketplace---Amazon Mechanical Turk---to compare $10$ photos of our institution
with regard to \emph{color}, \emph{sharpness} and \emph{landscape}.
The $135$ possible questions were also partitioned into $15$ tasks, each containing $9$ questions on a criterion.
We did not provide the ``skip'' option, because the questions require no special knowledge.
We did basic quality control by including in each task two validation questions that expect certain answers.
For instance, one such question asks the crowd to compare a colorful photo and a dull photo by criterion \emph{color}.
A crowdsourcer's input is discarded if their response to a validation question deviates from our expectation.
($5$ crowdsourcers failed on this.)
The parameters in Equation~(\ref{eq:outcome}) were set to be $k$$=$$10$ and $\theta$$=$$0.6$. (Thus each of the $15$ tasks was taken by
$10$ crowdsourcers that passed the validation.)
As Table~\ref{table:caseStudyResults} shows, the numbers of questions by various methods are similar to those in Case Study 1.
Based on the question outcomes, $3$ of the $10$ photos are Pareto-optimal.
Fig.\ref{fig:picDom} shows the Hasse diagram for the photos.

\begin{figure}[t]
    \centering
    \includegraphics[width=85mm]{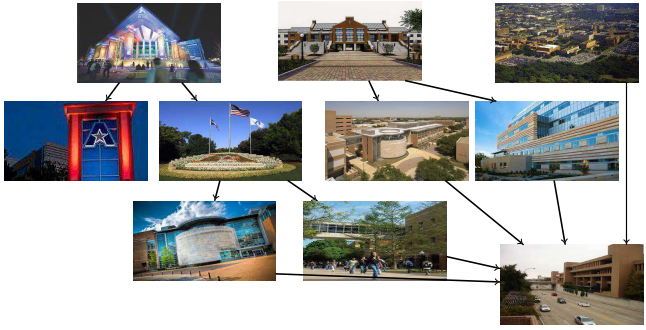}\vspace{-2mm}
    \caption{\small{Hasse diagram for the $10$ photos in Case Study 2.}}
    \label{fig:picDom}
\end{figure}

\section{Conclusions} 
This paper is the first study on how to use crowdsourcing to find Pareto-optimal objects when objects do not have explicit attributes and preference relations are strict partial orders.  The partial orders are obtained by pairwise comparison questions to the crowd.  It introduces an iterative question-selection framework that is instantiated into different methods by exploiting the ideas of candidate questions, macro-ordering and micro-ordering.  Experiment were conducted by simulations on large object sets and case studies were carried out using both human judges and a real crowdsourcing marketplace. The results exhibited not only orders of magnitude reductions in questions when compared with a brute-force approach, but also close-to-optimal performance from the most efficient method.

\small
\bibliographystyle{IEEEtran}
\bibliography{reference-crowd-pareto-opt}	

\begin{thebibliography}{10}
\providecommand{\url}[1]{#1}
\csname url@samestyle\endcsname
\providecommand{\newblock}{\relax}
\providecommand{\bibinfo}[2]{#2}
\providecommand{\BIBentrySTDinterwordspacing}{\spaceskip=0pt\relax}
\providecommand{\BIBentryALTinterwordstretchfactor}{4}
\providecommand{\BIBentryALTinterwordspacing}{\spaceskip=\fontdimen2\font plus
\BIBentryALTinterwordstretchfactor\fontdimen3\font minus
  \fontdimen4\font\relax}
\providecommand{\BIBforeignlanguage}[2]{{%
\expandafter\ifx\csname l@#1\endcsname\relax
\typeout{** WARNING: IEEEtran.bst: No hyphenation pattern has been}%
\typeout{** loaded for the language `#1'. Using the pattern for}%
\typeout{** the default language instead.}%
\else
\language=\csname l@#1\endcsname
\fi
#2}}
\providecommand{\BIBdecl}{\relax}
\BIBdecl

\bibitem{bucketorder-fagin-pods04}
R.~Fagin, R.~Kumar, M.~Mahdian, D.~Sivakumar, and E.~Vee, ``Comparing and
  aggregating rankings with ties,'' in \emph{PODS}, 2004, pp. 47--58.

\bibitem{chomicki2003preference}
J.~Chomicki, ``Preference formulas in relational queries,'' \emph{TODS}, 2003.

\bibitem{borzsony2001skyline}
S.~Borzsony, D.~Kossmann, and K.~Stocker, ``The skyline operator,'' in
  \emph{ICDE}, 2001, pp. 421--430.

\bibitem{chan2005stratified}
C.-Y. Chan, P.-K. Eng, and K.-L. Tan, ``Stratified computation of skylines with
  partially-ordered domains,'' in \emph{SIGMOD}, 2005, pp. 203--214.

\bibitem{sacharidis2009topologically}
D.~Sacharidis, S.~Papadopoulos, and D.~Papadias, ``Topologically sorted
  skylines for partially ordered domains,'' in \emph{ICDE}, 2009, pp.
  1072--1083.

\bibitem{sarkas2008categorical}
N.~Sarkas, G.~Das, N.~Koudas, and A.~K. Tung, ``Categorical skylines for
  streaming data,'' in \emph{SIGMOD}, 2008, pp. 239--250.

\bibitem{skylinepartialorder-zhang-vldb10}
S.~Zhang, N.~Mamoulis, D.~W. Cheung, and B.~Kao, ``Efficient skyline evaluation
  over partially ordered domains,'' in \emph{VLDB}, 2010.

\bibitem{lofi2013skyline}
C.~Lofi, K.~El~Maarry, and W.-T. Balke, ``Skyline queries in crowd-enabled
  databases,'' in \emph{EDBT}, 2013, pp. 465--476.

\bibitem{thurstone1927law}
L.~L. Thurstone, ``A law of comparative judgment,'' \emph{Psychological
  Review}, vol.~34, pp. 273--286, 1927.

\bibitem{kiessling2002foundations}
W.~Kie{\ss}ling, ``Foundations of preferences in database systems,'' in
  \emph{VLDB}, 2002, pp. 311--322.

\bibitem{pairwiseranking-chen-wism13}
X.~Chen, P.~N. Bennett, K.~Collins-Thompson, and E.~Horvitz, ``Pairwise ranking
  aggregation in a crowdsourced setting,'' in \emph{WSDM}, 2013.

\bibitem{humantopk}
V.~Polychronopoulos, L.~de~Alfaro, J.~Davis, H.~Garcia-Molina, and
  N.~Polyzotis, ``Human-powered top-k lists,'' in \emph{WebDB}, 2013.

\bibitem{davidson2013topkgroupby}
S.~B. Davidson, S.~Khanna, T.~Milo, and S.~Roy, ``Using the crowd for top-k and
  group-by queries,'' in \emph{ICDT}, 2013, pp. 225--236.

\bibitem{arrow1951social}
K.~J. Arrow, \emph{Social choice and individual values}, 1951.

\bibitem{liu-cikm09}
N.~N. Liu, M.~Zhao, and Q.~Yang, ``Probabilistic latent preference analysis for
  collaborative filtering,'' in \emph{CIKM}, 2009, pp. 759--766.

\bibitem{rendle-uai09}
S.~Rendle, C.~Freudenthaler, Z.~Gantner, and L.~Schmidt-Thieme, ``{BPR}:
  Bayesian personalized ranking from implicit feedback,'' in \emph{UAI}, 2009.

\bibitem{crowdranking-yi-hcomp13}
J.~Yi, R.~Jin, S.~Jain, and A.~K. Jain, ``Inferring users' preferences from
  crowdsourced pairwise comparisons: A matrix completion approach,'' in
  \emph{HCOMP}, 2013.

\bibitem{collaborativefiltering-goldberg-cacm92}
D.~Goldberg, D.~Nichols, B.~M. Oki, and D.~Terry, ``Using collaborative
  filtering to weave an information tapestry,'' \emph{CACM}, vol.~35, no.~12,
  1992.

\bibitem{ranknet-burges-icml05}
C.~Burges, T.~Shaked, E.~Renshaw, A.~Lazier, M.~Deeds, N.~Hamilton, and
  G.~Hullender, ``Learning to rank using gradient descent,'' in \emph{ICML},
  2005, pp. 89--96.

\bibitem{irsvm-cao-sigir06}
Y.~Cao, J.~Xu, T.-Y. Liu, H.~Li, Y.~Huang, and H.-W. Hon, ``Adapting ranking
  {SVM} to document retrieval,'' in \emph{SIGIR}, 2006, pp. 186--193.

\bibitem{lambdarank-burges-nips06}
C.~J.~C. Burges, R.~Ragno, and Q.~V. Le, ``{Learning to Rank with Nonsmooth
  Cost Functions},'' in \emph{NIPS}, 2006, pp. 193--200.

\bibitem{learntorank-liu-09}
T.-Y. Liu, ``Learning to rank for information retrieval,'' \emph{Foundations
  and Trends in Information Retrieval}, vol.~3, no.~3, pp. 225--331, Mar. 2009.

\bibitem{braverman-soda08}
M.~Braverman and E.~Mossel, ``Noisy sorting without resampling,'' in
  \emph{SODA}, 2008, pp. 268--276.

\bibitem{jamieson-nips11}
K.~G. Jamieson and R.~D. Nowak, ``Active ranking using pairwise comparisons,''
  in \emph{NIPS}, 2011, pp. 2240--2248.

\bibitem{ailon-nips11}
N.~Ailon, ``Active learning ranking from pairwise preferences with almost
  optimal query complexity,'' in \emph{NIPS}, 2011, pp. 810--818.

\bibitem{ailon-jmlr12}
------, ``An active learning algorithm for ranking from pairwise preferences
  with an almost optimal query complexity,'' \emph{Journal of Machine Learning
  Research}, vol.~13, no.~1, pp. 137--164, Jan. 2012.

\bibitem{negahban-nips12}
S.~Negahban, S.~Oh, and D.~Shah, ``Iterative ranking from pair-wise
  comparisons,'' in \emph{NIPS}, 2012, pp. 2483--2491.

\bibitem{quality-ipeirotis-2010}
P.~G. Ipeirotis, F.~Provost, and J.~Wang, ``Quality management on amazon
  mechanical turk,'' in \emph{HCOMP}, 2010, pp. 64--67.

\bibitem{rorvig-jasis90}
M.~E. Rorvig, ``The simple scalability of documents,'' \emph{Journal of the
  American Society for Information Science}, vol.~41, no.~8, 1990.

\bibitem{carterette-ecir08}
B.~Carterette, P.~N. Bennett, D.~M. Chickering, and S.~T. Dumais, ``Here or
  there: Preference judgments for relevance,'' in \emph{ECIR}, 2008, pp.
  16--27.

\end{thebibliography}
\end{document}